\documentclass[11pt,letterpaper]{article}

\usepackage{amsmath,amsfonts,bm}
\usepackage{amssymb,amsthm}

\def\1{\bm{1}}

\def\vtheta{{\bm{\theta}}}
\def\va{{\bm{a}}}

\def\ve{{\bm{e}}}

\def\vm{{\bm{m}}}

\def\vx{{\bm{x}}}

\def\vz{{\bm{z}}}

\def\mA{{\bm{A}}}

\def\mD{{\bm{D}}}

\def\mI{{\bm{I}}}

\def\mU{{\bm{U}}}
\def\mV{{\bm{V}}}

\DeclareMathAlphabet{\mathsfit}{\encodingdefault}{\sfdefault}{m}{sl}
\SetMathAlphabet{\mathsfit}{bold}{\encodingdefault}{\sfdefault}{bx}{n}

\def\gD{{\mathcal{D}}}

\def\gL{{\mathcal{L}}}

\def\gN{{\mathcal{N}}}
\def\gO{{\mathcal{O}}}

\def\gU{{\mathcal{U}}}

\newcommand{\E}{\mathbb{E}}

\newcommand{\R}{\mathbb{R}}

\newcommand{\vepsilon}{\mathbf{\epsilon}}
\newcommand{\0}{\mathbf{0}}
\newcommand{\mdelta}{\mathbf{\delta}}
\newcommand{\bnu}{\Bar{\nu}}
\newcommand{\aat}{\mA\mA^T}
\newcommand{\lxt}{\overleftarrow{\vx_T}}
\newcommand{\leps}{\overleftarrow{\vepsilon}}
\newcommand{\reps}{\overrightarrow{\vepsilon}}
\newcommand{\gauss}{\gN\left(\0,\mI \right)}
\newcommand{\dm}{\mD(\vm)}
\newcommand{\dom}{\mD(\1 - \vm)}

\def\lxz{\overleftarrow{\vx_0}}
\def\rxz{\overrightarrow{\vx_0}}
\def\lxo{\overleftarrow{\vx_1}}
\def\rxo{\overrightarrow{\vx_1}}
\def\lzz{\overleftarrow{\vz_0}}
\def\rzz{\overrightarrow{\vz_0}}
\def\vzz{\vz_0}

\usepackage[ruled,linesnumbered]{algorithm2e}
\SetKwComment{Comment}{/* }{ */}
\usepackage{wasysym}
\usepackage{mathtools}
\usepackage{authblk}
\usepackage{algorithmic}
\usepackage{subcaption}
\usepackage{wrapfig}
\usepackage{resizegather}

\usepackage[margin=0.9in]{geometry}
\usepackage{lmodern}
\usepackage[T1]{fontenc}
\usepackage[utf8]{inputenc}

\allowdisplaybreaks
\usepackage{bbm}
\usepackage{setspace} 

\usepackage{url}
\usepackage[svgnames]{xcolor}
\definecolor{cblue}{rgb}{0,0.5,50}
\definecolor{cmaroon}{rgb}{139,0,0}

\usepackage[style=alphabetic, maxbibnames=15, maxcitenames=15, natbib=true,
  maxalphanames=10, backend=biber, sorting=nty, backref=true]{biblatex}
\DefineBibliographyStrings{english}{%
  backrefpage = {page},
  backrefpages = {pages},
}

\usepackage{doi}
\usepackage{tikz} 
\usepackage{pgf,pgfplots}
\pgfplotsset{compat=1.14}
\usepgfplotslibrary{fillbetween}
\usepackage{hyphenat}
\usepackage{appendix}
\usepackage{thm-restate}

\usepackage[capitalise,sort]{cleveref}

\usepackage{dsfont}
\usepackage{microtype,bm}
\usepackage{booktabs}
\usepackage{colortbl}
\usepackage{xcolor}
\usepackage{nicefrac}

\newtheorem{theorem}{Theorem}

\newtheorem{corollary}[theorem]{Corollary}

\newtheorem{assumption}[theorem]{Assumption}
\crefname{assumption}{Assumption}{Assumptions}

\hypersetup{
    colorlinks = True,
    linkcolor=cmaroon,
    citecolor=cblue,
    urlcolor=magenta
}

\usepackage[backend=biber]{biblatex}

\author{Litu Rout\thanks{litu.rout@utexas.edu}}
\author{Advait Parulekar\thanks{advaitp@utexas.edu}}
\author{Constantine Caramanis\thanks{constantine@utexas.edu}}
\author{Sanjay Shakkottai\thanks{sanjay.shakkottai@utexas.edu}}
\affil{The University of Texas at Austin}

\title{A Theoretical Justification for Image Inpainting using \\Denoising Diffusion Probabilistic Models}

\date{}

\begin{document}

\maketitle

\begin{abstract}
We provide a theoretical justification for sample recovery using diffusion based image inpainting in a linear model setting. While most inpainting algorithms require retraining with each new mask, we prove that diffusion based inpainting generalizes well to unseen masks without retraining. We analyze a recently proposed popular diffusion based inpainting algorithm called RePaint \cite{lugmayr2022repaint}, and show that it has a bias due to misalignment that hampers sample recovery even in a two-state diffusion process. Motivated by our analysis, we propose a modified RePaint algorithm we call RePaint$^+$ that provably recovers the underlying true sample and enjoys a linear rate of convergence. It achieves this by rectifying the misalignment error present in drift and dispersion of the reverse process. To the best of our knowledge, this is the first linear convergence result for a diffusion based image inpainting algorithm. 
\end{abstract}

\section{Introduction}
\label{sec:intro}
We study the mathematical principles that underlie the empirical success of image inpainting using Denoising Diffusion Probabilistic Models (DDPMs)~\cite{sohl2015deep, ho2020denoising, song2020score}, which are the backbone of large-scale generative models including DALL-E~\cite{ramesh2021zero, ramesh2022hierarchical}, Imagen~\cite{saharia2022photorealistic}, and Stable Diffusion~\cite{Rombach_2022_CVPR}. The goal of image inpainting is to reconstruct missing parts that are semantically consistent with the known portions of an image. A major challenge in this task is to generate parts that are consistent with the available parts of the image; generative modeling is the key tool for this inpainting goal. 

\begin{figure}[t]
\begin{center}
\centerline{\includegraphics[width=0.45\columnwidth]{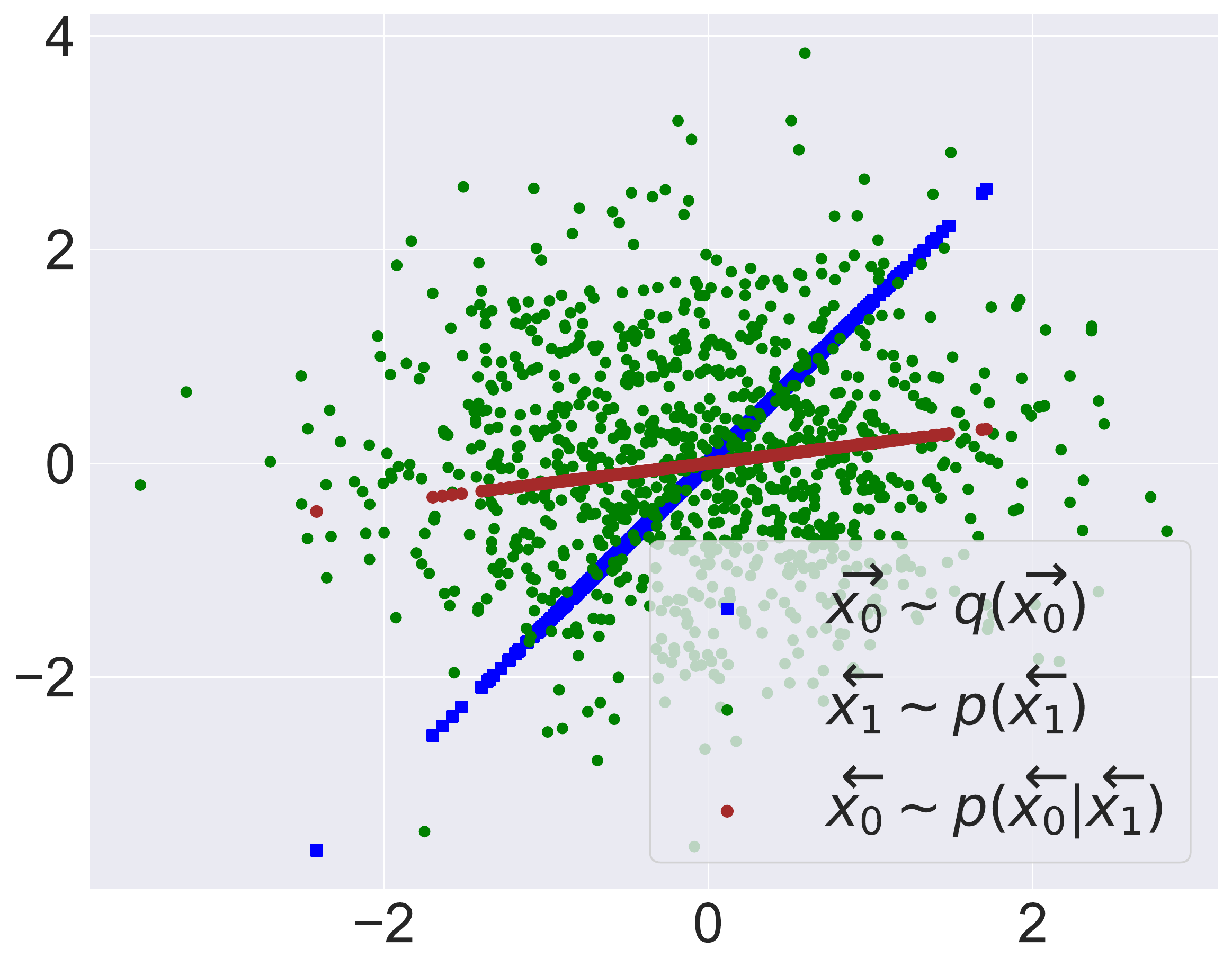}}
\caption{Example demonstrating the bias in RePaint~\cite{lugmayr2022repaint}. Starting from the Gaussian prior $\protect \lxo \protect$ (green circles), reverse SDE as proposed by RePaint generates $\protect\lxz\protect$ (brown circle) that matches with the true data $\protect\rxz\protect$ (blue square) in known coordinates, but differs in the inpainted region. In the figure, the blue squares along the $\protect\sim 56^\circ\protect$ line represents the true samples, whereas the brown circles along the $\protect\sim 10^\circ\protect$ line represents the samples recovered by RePaint. Note that the recovered samples match the true samples along the x-coordinate (known data), but have a bias along the y-coordinate (missing data).
}
\label{fig:repaint}
\end{center}
\vskip -0.4in
\end{figure}

Variational-Auto Encoders~(VAEs)~\cite{kingma2013auto} and Generative Adversarial Networks (GANs)~\cite{goodfellow2020generative} have been the basis for many successful inpainting techniques in recent years. However, with the advent of Score based Generative Models (SGMs)~\cite{song2019generative} or DDPMs~\cite{sohl2015deep,ho2020denoising}, current focus has shifted towards an alternate paradigm of image inpainting. In this paradigm, most methods fall into one of two categories. In the first category, one learns a diffusion process specific to a downstream task, such as inpainting or super-resolution~\cite{Whang_2022_CVPR,saharia2022palette}. In the second, a general purpose diffusion based generative model is learned and the diffusion process is guided in the inference phase catering to the downstream task~\cite{jalal2021robustmri,choi2021ilvr,song2021solving,daras2022score,kawar2022denoising}. In this paper, we analyze an inpainting approach~\cite{lugmayr2022repaint} that falls in the second category. We show that the generative prior of DDPM is sufficient to fully reconstruct missing parts of an image.

DDPM represents a class of generative models that learn to diffuse a clean image into tractable noise and then follow a reverse Markov process to produce a clean image by progressively denoising pure noise~\cite{sohl2015deep,ho2020denoising,song2020score} (see \wasyparagraph{\ref{sec:background}). Owing to their high expressive power, these models show appealing results in conditional/unconditional image generation~\cite{ho2020denoising,song2020score,dhariwal2021diffusion,karras2022elucidating}, text-to-image synthesis~\cite{Rombach_2022_CVPR}, time series modeling~\cite{tashiro2021csdi}, audio synthesis~\cite{kong2021diffwave}, image-to-image translation~\cite{saharia2022palette}, controllable text generation~\cite{li2022diffusionlm}, and image restoration~\cite{kawar2022denoising}. For restoration tasks such as image inpainting, DDPMs are interesting because they easily adapt to unknown tasks, such as new masks without having to go through the entire retraining process. This gives DDPM an edge over prior restoration techniques based upon GANs~\cite{wang2018esrgan,yu2018generative} and VAEs~\cite{razavi2019generating,peng2021generating}.

However, one challenging aspect limiting their pervasive usage stems from the choice of hyper-parameters in the reverse Markov process. Although the forward diffusion process can be explicitly computed beforehand, the reverse process is computationally very expensive as it requires sampling at every intermediate state. Furthermore, popular diffusion based inpainting algorithms, such as RePaint~\cite{lugmayr2022repaint} require additional resampling at each of these states to harmonize inpainted parts with the rest of the image. These resampling steps increase the computational burden at the cost of semantically meaningful reconstruction. Also, Figure~\ref{fig:repaint} shows that RePaint generates \textit{biased} samples that hamper perfect recovery. Therefore, it is becoming increasingly important to address these bottlenecks to facilitate their successful real-world deployment. A persistent challenge in this respect stems from the fact that despite the remarkable progress of diffusion based inpainting, our theoretical understanding remains in its early stages.

To address these issues,
one emerging line of theoretical research aims to provide convergence guarantees for Score-based Generative Models (SGMs) that constitute DDPMs~\cite{chen2022sampling,lee2022convergence,lee2022convergenceGeneral,chen2022improved}. A crucial assumption of these approaches is that the target distribution $q(\rxz)$ assumes a density with respect to Lebesgue measure. Another line of work~\cite{debortoli2022convergence,pidstrigach2022scorebased} studies convergence of SGMs under the famous manifold hypothesis~\cite{tenenbaum2000global,fefferman2016testing}. Starting from a prior $p(\overleftarrow{\vx_T})$, these lines of work analyze convergence of the reverse process $p_{\vtheta}(\lxz)$ to the target distribution $q(\rxz)$. Different from these lines of research, our paper focuses on the convergence of resampling step used to harmonize the inpainted image (\wasyparagraph{\ref{sec:theory}).

The assumption that the target distribution admits a density according Lebesgue measure indicates that $q(\rxz)>0$ for all $\rxz \in \R^d$. In other words, every $\rxz$ is a possible sample drawn from $q(\rxz)$, which contradicts the fact that most natural images reside on a low dimensional manifold with compact support~\cite{tenenbaum2000global,fefferman2016testing}. This assumption would require a target distribution over digits to allocate nonzero probability mass to very unlikely samples, such as human faces, animals, and bedroom scenes. Although the prior analysis offers some insights, it is of little practical significance in this manifold setting~\cite{pmlr-v162-kim22i}. As a step towards circumventing this issue, \citet{pidstrigach2022scorebased,debortoli2022convergence} prove that SGMs learn distributions supported on a low dimension substructure, satisfying the manifold hypothesis.

For inpainting, however, we need an additional structure on the data distribution similar to prior works in related disciplines~\cite{bora2017compressed,dhar2018modeling,jalal2020robust,jalal2021robustmri}. This is because our goal is not just to show minimum discrepancy between $q(\rxz)$ and $p_\vtheta(\lxz)$ with respect to some divergence $\gD$, i.e., $\gD\left(p_\vtheta(\lxz), q(\rxz)\right)\leq \epsilon$,
but also to prove sample wise convergence, i.e., $\left\|\lxz - \rxz \right\| \leq \epsilon$. It is worth mentioning that the problem of sample recovery is well studied in optimization literature~\cite{tibshirani1996regression,4407762,5454399,1217267,bora2017compressed}. The focus of our paper, instead, is to use this well understood setting to provide a
theoretical justification for sample recovery using a DDPM-based non-optimization technique~\cite{lugmayr2022repaint}. In this technique, we leverage the expressive power of diffusion based generative models~\cite{ho2020denoising} without attempting to directly solve a constrained optimization problem. Following conventional wisdom~\cite{bora2017compressed,jalal2020robust,jalal2021robustmri}, we implant a specific model on the data generating distribution that satisfies the manifold hypothesis, and show that the learning algorithm recovers this model by detecting the underlying substructure: see Figure~\ref{fig:repaintplus}.

\begin{figure}[t]
\begin{center}
\centerline{\includegraphics[width=0.45\columnwidth]{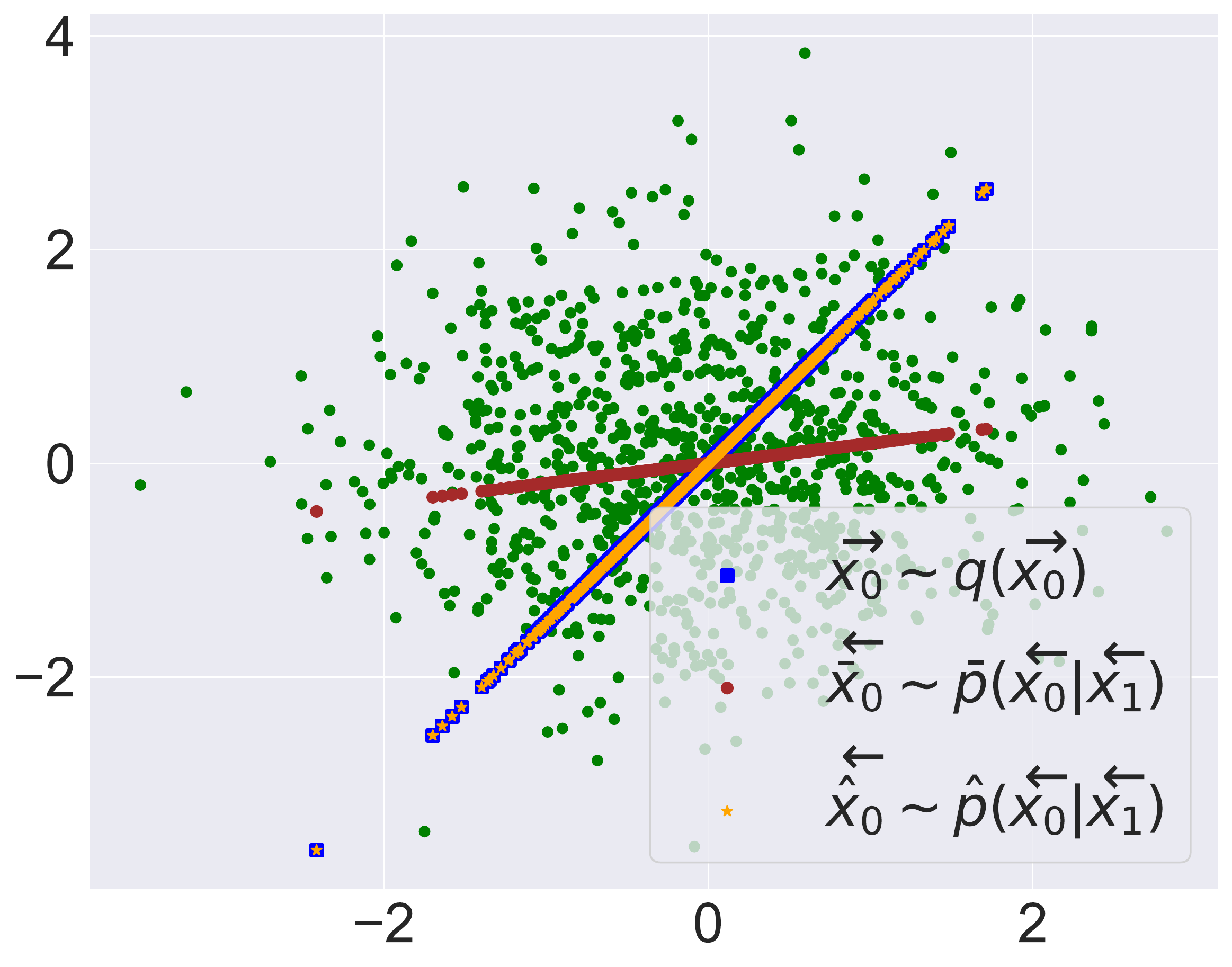}}
\caption{\texorpdfstring{Comparison between RePaint ($\protect\overleftarrow{\Bar{\vx}_0}\protect$) and RePaint$\protect^+\protect$ ($\protect\overleftarrow{\hat{\vx}_0}\protect$). RePaint$\protect^+\protect$ discovers $\protect\epsilon\protect$-accurate solutions whereas RePaint suffers from the misalignment bias. Green circles  indicate Gaussian prior $\protect\lxo\protect$ of the reverse Markov process. In the figure, the blue squares along the $\protect\sim 56^\circ\protect$ line represent the true samples. As before, the brown circles along the $\protect\sim 10^\circ\protect$ line represents the samples recovered by RePaint (this has a bias) and the orange stars along the $\protect\sim 56^\circ\protect$ line represent the samples recovered by Repaint$\protect^+\protect$. Note that the recovered samples by RePaint$\protect^+\protect$ has no observable bias (recovered samples are overlapping on the true samples).}{See caption in the published version.}}
\label{fig:repaintplus}
\end{center}
\vskip -0.3in
\end{figure}

\subsection{Contributions}
\label{subsec:contrib}
Our main contribution is to provide a theoretical justification for diffusion based image inpainting in a linear model setting. By analyzing diffusion over two states,
our analysis explains previously not understood phenomena. Importantly, we derive algorithmic insights which, as we demonstrate later,
deliver improvement beyond the two-state processes. 

First, unlike prior inpainting methods, we prove that diffusion based image inpainting easily adapts to each new mask without retraining (\textbf{Theorem~\ref{thm:inp-d-dim}} in \wasyparagraph{\ref{sec:theory}). 

Next, inspired by the empirical success of a recently proposed diffusion based image inpainting method~\cite{lugmayr2022repaint}, we analyze its theoretical properties in a system of two-state diffusion processes. We observe that there exists a bias due to misalignment that hampers perfect recovery: see Figure~\ref{fig:repaint}. Our analysis motivates us to rectify this misalignment that helps eliminate the bias.  We refer to this method as RePaint$^+$. We provide a simplified version of RePaint$^+$ in \textbf{Algorithm~\ref{alg:repaint-inference-special}} and a general version in \textbf{Algorithm~\ref{alg:repaint-inference-general}}. We conduct toy experiments in Appendix~\ref{sec:exps}. 

In the linear model setting, we derive a closed-form solution of the generative model using the transition kernels of DDPM (\textbf{Theorem~\ref{thm:gen-d-dim}}}). Using this solution, we prove that RePaint$^+$ enjoys a linear rate of convergence in the resampling phase (\textbf{Theorem~\ref{thm:inp-d-dim}}). This allows us to appropriately choose the number of resampling rounds in practice. We generalize this notion of convergence to a setting where the generative model is approximate (\textbf{Theorem~\ref{thm:inp-d-dim-noise}}). Further, we justify the benefits of resampling over slowing down the diffusion process in the context of inpainting (\wasyparagraph{\ref{subsec:main-slow-jump}}).

\textbf{Notation:} We denote by bold upper-case letter $\mA$ a matrix, bold lower-case letter $\vx$ a vector, and normal lower-case letter $x$ a scalar. $\mI_k$ denotes a $k\times k$ identity matrix. The set of integers $\{1,\dots,N \}$ are captured in $[N]$. Element-wise product is represented by $\odot$. The operator $\mD\left(\vx\right)$ diagonalizes a vector $\vx$.
$\left\|\mA\right\|$ denotes the spectral norm and $\left\|\mA\right\|_F$, the Frobenius norm of a matrix $\mA$. For a vector, $\left\|\vx\right\|$ denotes its Euclidean norm. $\gU\left(\cdots \right)$ denotes uniform distribution.

\section{Background on Diffusion Models}
\label{sec:background}
Like other generative models, such as GANs~\cite{goodfellow2020generative}, VAEs~\cite{kingma2013auto}, and Flow~\cite{dinh2017density}, Score-based generative models (SGM)~\cite{song2019generative} learn to sample from  an unknown distribution given a set of samples drawn from this distribution. In this section, we first briefly introduce key ingredients of SGM. Then, we detail the DDPM interpretation of SGM in \wasyparagraph{\ref{subsec:ddpm}}, which holds the foundation of RePaint$^+$.

\subsection{Score-based Generative Model}
\label{subsec:sgm}
The central part of SGM consists of two Stochastic Differential Equations (SDEs). The \textit{forward SDE} (\ref{eq:fwd-sde}), an Ornstein-Uhlenbeck (OU) process in its simplest form, transforms the data distribution $q(\rxz)$ to a reference distribution, which is $\gN\left( \0,\mI\right)$ in most cases. Here, a sample $\overrightarrow{\vx_t}$ at time $t$ follows:
\begin{align}
    d \overrightarrow{\vx_t} = -\frac{1}{2} \overrightarrow{\vx_t}dt + dW_t;~\rxz \sim q(\rxz). 
    \label{eq:fwd-sde}
\end{align}
With time rescaling parameters $\alpha_t$ and $\beta_t$, where $\alpha_t = \int_{0}^{t}\beta_s ds$, and normalized time $t = \{0,\cdots,1\}$, the transition kernel of (\ref{eq:fwd-sde}) becomes:
\begin{align*}
    q\left(\overrightarrow{\vx_t}\mid \rxz\right) = \gN\left(\overrightarrow{\vx_t};\exp{(-\frac{\alpha_t}{2} )\rxz, \left( 1 - \exp{(-\alpha_t)}\right)\mI
    }\right),
\end{align*}
 which leads to the well known form of forward SDE:
 \begin{align*}
     d \overrightarrow{\vx_t} = -\frac{1}{2} \beta_t\overrightarrow{\vx_t} + \sqrt{\beta_t}d\mathbf{B}_t;~\rxz \sim q(\rxz).
 \end{align*}
Here, $\left \{ d\mathbf{B}_t \right \}_{t\geq 0}$ represents the standard Brownian motion in $\R^d$. The \textit{reverse SDE} has a form similar to the forward SDE with time reversal:
\begin{align*}
    d\overleftarrow{\vx_t} = \frac{1}{2}\beta_{1-t} \overleftarrow{\vx_t} + \beta_{1-t} \nabla \log q(\overrightarrow{\vx_t}|\rxz) + \sqrt{\beta_{1-t}} d\mathbf{B}_t,
\end{align*}
where $\overleftarrow{\vx_T} \sim p\left(\overleftarrow{\vx_T} \right)\coloneqq \gN\left( \0,\mI\right)$. Usually, the Euler-Maryuama discretization scheme is employed while implementing these SDEs in practice. Since we do not have access to $q(\rxz)$, a neural network $s_\vtheta\left(\overrightarrow{\vx_t},t\right)$ is used to approximate the score function $\nabla \log q(\overrightarrow{\vx_t}|\rxz)$. Training is performed by minimizing a score-matching objective~\cite{hyvarinen2005estimation,vincent2011connection,song2019generative}:
\begin{align*}
    \min_{\theta} \mathop{\E}_{\substack{\rxz \sim q\left(\rxz \right) \\\overrightarrow{\vx_t} \sim q\left(\overrightarrow{\vx_t}|\rxz\right)}} \left[ \left\| \nabla \log q(\overrightarrow{\vx_t}|\rxz) -  s_\vtheta\left(\overrightarrow{\vx_t},t\right)\right\|^2 \right].
\end{align*}
Next, we discuss DDPM and its connection with SGM.

\subsection{Denoising Diffusion Probabilistic Model}
\label{subsec:ddpm}
 Diffusion model~\citep{sohl2015deep,ho2020denoising} is an emerging class of generative models that share strikingly similar properties with score-based generative models~\cite{song2019generative}. It consists of two stochastic processes. First, the \textit{forward process (diffusion process)} gradually adds Gaussian noise to an image according to a variance schedule. This is a Markov chain with stationary distribution typically set to a tractable distribution that is easy to sample from, e.g., a standard Gaussian $\gN\left(\0,\mI\right)$. Second, the \textit{reverse process (denoising process)} learns to gradually denoise a sample drawn from the tractable distribution. The denoising process continues until it produces a high-quality image from the original data distribution. 

\subsubsection{Forward Process}
\label{subsubsec:fwd_process}

Let $\overrightarrow{\vx_0} \sim q(\rxz)$ denote a sample in $\R^d$ drawn from the data distribution $q(\rxz)$. For discrete time steps $1\leq t \leq T$, $q(\vx_t)$ represents the distribution at time $t$. In the \textit{forward process}, the diffusion takes place according to a \textit{fixed} Gaussian transition kernel, i.e.,
\begin{align}
\label{eq:forward}
q\left(\overrightarrow{\vx_{1:T}}|\rxz \right) \coloneqq \prod_{t=1}^{T} q(\overrightarrow{\vx_t}|\overrightarrow{\vx_{t-1}});
~q(\overrightarrow{\vx_t}|\overrightarrow{\vx_{t-1}}) \coloneqq \mathcal{N}\left ( \overrightarrow{\vx_t}; \sqrt{1-\beta_t}\overrightarrow{\vx_{t-1}}, \beta_t \mI_d \right),
\end{align} 
where $\beta_1,\dots,\beta_T$ denote a deterministic variance schedule. Thus, a sample from $q(\overrightarrow{\vx_t}|\overrightarrow{\vx_{t-1}})$ is given by $\overrightarrow{\vx_t} = \sqrt{1 - \beta_t} \overrightarrow{\vx_{t-1}} + \sqrt{\beta_t}  \mathbf{\vepsilon}$, where $\vepsilon \sim \gN\left(\mathbf{0},\mI_d\right)$. The first term is called drift and the second, dispersion. An important property of the Gaussian diffusion process is that it has a simple form when sampling from any intermediate time steps. For $\alpha_t \coloneqq 1 - \beta_t$ and $\Bar{\alpha}_t \coloneqq \prod_{s=1}^{t} \alpha_s$, the conditional probability at time $t$ becomes: 
\begin{align*}
    q\left(\overrightarrow{\vx_t}|\overrightarrow{\vx_{t-1}}\right) = \gN\left(\overrightarrow{\vx_t}; \sqrt{\Bar{\alpha}_t}\overrightarrow{\vx_0}, \left(1-\Bar{\alpha}_t \right)\mI_d\right).
\end{align*}

\subsubsection{Reverse Process}
\label{subsubsec:rev_process}
The \textit{reverse process} is a Markov chain initialized at the stationary distribution of the \textit{forward process}, i.e., $p(\overleftarrow{\vx_T}) = \gN\left(\0,\mI_d\right)$. Unlike the forward process, the reverse process $p_\theta\left(\overleftarrow{\vx_{0:T}}\right)$ is generated from a \textit{learned} Gaussian kernel, i.e.,
\begin{align}
p_\theta\left ( \overleftarrow{\vx_{0:T}} \right ) = p(\overleftarrow{\vx_T}) \prod_{t=1}^{T} p_\theta\left ( \overleftarrow{\vx_{t-1}}|\overleftarrow{\vx_{t}} \right );
~p_\theta\left ( \overleftarrow{\vx_{t-1}}|\overleftarrow{\vx_{t}} \right ) \coloneqq \gN\left(\overleftarrow{\vx_{t-1}}; \mu_\theta\left(\overleftarrow{\vx_t}, t \right), \Sigma_\theta \left(\overleftarrow{\vx_t}, t\right)\right),
\label{eq:reverse}
\end{align}
where $\mu_\theta \left(\overleftarrow{\vx_t}, t\right)$ and $\Sigma_{\theta}\left(\overleftarrow{\vx_t}, t\right)$ are neural networks parameterized by $\theta$ to predict the mean and the variance of $p_\theta\left ( \overleftarrow{\vx_{t-1}}|\overleftarrow{\vx_{t}} \right )$, respectively.

\subsubsection{Training and Inference}
\label{subsec:train_and_infer}
DDPM aims to maximize the likelihood of a sample generated by the reverse process. As per equation~(\ref{eq:reverse}), the probability assigned to such a sample is obtained by marginalizing over the remaining random variables, i.e., $\overleftarrow{\vx_1}, \overleftarrow{\vx_2}, \cdots, \overleftarrow{\vx_T}$ as denoted by the following expression:   
\begin{align*}
    p(\overleftarrow{\vx_0}) = \int p(\overleftarrow{\vx_{0:T}}) d\overleftarrow{\vx_{1:T}},
\end{align*}
which is not easy to compute. It becomes tractable by considering relative probability between the forward and the reverse processes \cite{jarzynski1997equilibrium,sohl2015deep},  leading to the variational lower bound,
\begin{align}
&\mathbb{E}\left [ -\log p_\theta\left ( \overleftarrow{\vx_0} \right ) \right ] \leq \mathbb{E}_q \left [ -\log \frac{p_\theta\left ( \overleftarrow{\vx_{0:T}} \right )}{q\left ( \overrightarrow{\vx_{1:T}}|\overrightarrow{\vx_0} \right )} \right ]
= \mathbb{E}_q\left [ -\log p\left ( \overleftarrow{\vx_T} \right ) - \sum_{t\geq1} \log\frac{p_\theta\left ( \overleftarrow{\vx_{t-1}}| \overleftarrow{\vx_t} \right )}{q\left ( \overrightarrow{\vx_{t}}|\overrightarrow{\vx_{t-1}} \right )} \right ].
\label{eq:vlbo}
 \end{align}
This objective is further simplified to
\begin{align}
\label{eq:vlbo_kl}
 \mathbb{E}_q \big[\gD_{KL}\left ( q\left ( \overrightarrow{\vx_T}|\overrightarrow{\vx_0} \right ) \parallel p \left ( \overleftarrow{\vx_T} \right ) \right ) + \sum_{t>1}\gD_{KL}\left ( q\left ( \overrightarrow{\vx_{t-1}}|\overrightarrow{\vx_t}, \overrightarrow{\vx_0} \right ) \parallel p_\theta \left ( \overleftarrow{\vx_{t-1}}| \overleftarrow{\vx_{t}} \right ) \right ) - \log p_\theta\left ( \overleftarrow{\vx_0}|\overleftarrow{\vx_1} \right )\Big],
\end{align}
where $q\left(\overrightarrow{\vx_{t-1}}| \overrightarrow{\vx_t},\overrightarrow{\vx_0}\right) = \mathcal{N}\left ( \overrightarrow{\vx_{t-1}}; \tilde{\mu}_t\left ( \overrightarrow{\vx_t},\overrightarrow{\vx_0} \right ), \tilde{\beta}_t \mI_d \right )$, $\tilde{\mu}_t\left ( \overrightarrow{\vx_t},\overrightarrow{\vx_0} \right ) \coloneqq \frac{\sqrt{\Bar{\alpha}_{t-1} }\beta_t}{1-\Bar{a}_t}\overrightarrow{\vx_0} + \frac{\sqrt{\alpha}_t\left ( 1-\Bar{\alpha}_{t-1} \right )}{1-\Bar{\alpha}_t} \overrightarrow{\vx_t }$ and $ \tilde{\beta}_t \mI_d \coloneqq \frac{1-\Bar{\alpha}_{t-1}}{1-\Bar{\alpha}_t} \beta_t $. Equation~(\ref{eq:vlbo_kl}) contains \textit{three} crucial terms. The \textit{first} term measures the divergence between the stationary distribution of the forward process conditioned on a clean sample $\overrightarrow{\vx_0} \sim q\left(\overrightarrow{\vx_0}\right)$ and the initial distribution of the reverse process. The divergence is negligible in practice, thanks to the exponential convergence of OU processes. This is ignored as there are no trainable parameters, $\vtheta$.

The \textit{second} term measures the divergence between the forward and the reverse process at intermediate time steps, $1<t\leq T$. Since the conditional probabilities are Gaussians, this
can be explicitly computed as:
\begin{align}
    \label{eq:vlbo_mu}
\mathbb{E}_q\left [ \frac{1}{2\beta_t^2} \left \| \tilde{\mu}_t\left ( \overrightarrow{\vx_t}, \overrightarrow{\vx_0} \right ) - \mu_\theta\left ( \overleftarrow{\vx_t}, t \right ) \right \|^2
\right ] + Constant,
\end{align} 
where $p_\theta\left(\overleftarrow{\vx_{t-1}}|\overleftarrow{\vx_t} \right) = \gN\left( \overleftarrow{\vx_{t-1}}; \mu_\theta\left(\overleftarrow{\vx_t}, t\right), \beta_t\mI\right)$. In the training phase, $\overleftarrow{\vx_t}$ is replaced with $\overrightarrow{\vx_t}$ since the \textit{first} term in (\ref{eq:vlbo_kl}) is negligible.  Further, reparameterization of the posterior mean as $\mu_\theta\left(\overrightarrow{\vx_t}, t\right) = \frac{1}{\sqrt{\alpha_t}}\left ( \overrightarrow{\vx_t} - \frac{\beta_t}{\sqrt{1-\Bar{\alpha}_t}}~\vepsilon_\theta\left ( \overrightarrow{\vx_t}, t \right ) \right )$ yields better results and a simplified loss~\cite{ho2020denoising}:
\begin{align}
\label{eq:vlbo_reparam}
    \mathbb{E}_{\vx_0,\vepsilon}\left [ \frac{1}{2\alpha_t\left ( 1-\Bar{\alpha}_t \right )} \left \| \vepsilon - \vepsilon_\theta\left (\sqrt{\Bar{\alpha}_t}~\overrightarrow{\vx_0} + \sqrt{1-\Bar{\alpha}_t} ~\vepsilon, t \right ) \right \|^2
\right ].
\end{align}
During inference, for $\vepsilon \sim \gN\left( 
0,\mI\right)$, a sample is generated by $\overleftarrow{\vx_{t-1}} = \frac{1}{\sqrt{\alpha_t}}\left ( \overleftarrow{\vx_t} - \frac{\beta_t}{\sqrt{1-\Bar{\alpha}_t}} ~\vepsilon_\theta\left ( \overleftarrow{\vx_t}, t \right ) \right ) + \sqrt{\beta_t}~\vepsilon$.

The \textit{third} term is the standard maximum likelihood estimator.
With this reparameterization, the problem of learning probability is converted to a practically implementable minimum mean squared error (MMSE) problem. The goal of training DDPMs is to obtain an optimal $\theta^*$ that solves the MMSE problem with posterior mean~(\ref{eq:vlbo_mu}) or noise~(\ref{eq:vlbo_reparam}).

During inference, the reverse process is initialized at the stationary distribution of the forward process, usually $\gN\left(\0,\mI\right)$. Then, the reverse Gaussian transition kernel is followed using $\mu_{\vtheta^*}\left(\overleftarrow{\vx_t},t \right)$. In the final step, $\mu_{\vtheta^*}\left(\overleftarrow{\vx_1},1\right)$ is displayed without noise, which is a sample from the original data distribution. Ignoring some constant factors, the forward and reverse \textit{SDEs} of SGM and the forward and reverse \textit{processes} of DDPM are idential if we use Gaussian transition kernel. Throughout this paper, we interchangeably use these terms.
\vspace{-0.2in}
\begin{algorithm}[tb]
   \caption{DDPM Training}
   \label{alg:ddpm-train}
\begin{algorithmic}
   \STATE {\bfseries Input:}  Initialized weights $\vtheta$, Stepsize $\eta$
   \REPEAT
   \STATE Draw $\rxz \sim q(\rxz)$, $\epsilon \sim \gN\left( \0,\mI_d\right)$,
   \STATE $t\sim Uniform\left(\{1,\dots,T \}\right)$
   \STATE Loss (\ref{eq:vlbo_mu}): $\gL\left(\vtheta\right) = \left \| \tilde{\mu}_t\left ( \overrightarrow{\vx_t}, \overrightarrow{\vx_0} \right ) - \mu_\theta\left ( \overleftarrow{\vx_t}, t \right ) \right \|^2$  or \\ (\ref{eq:vlbo_reparam}): $\gL\left(\vtheta\right) = \left \| \vepsilon - \vepsilon_\theta\left (\sqrt{\Bar{\alpha}_t}~\overrightarrow{\vx_0} + \sqrt{1-\Bar{\alpha}_t} ~\vepsilon, t \right ) \right \|^2$ 
   \STATE Run gradient descent $\vtheta_{t+1} = \vtheta_{t} - \eta\nabla \gL\left( \vtheta\right)$
   \UNTIL{Convergence}
  \STATE {\bfseries Output:} Trained weights $\vtheta$
\end{algorithmic}
\end{algorithm}
\begin{algorithm}[tb]
   \caption{RePaint\textcolor{cyan}{$^+$} Inference Special Case}
   \label{alg:repaint-inference-special}
\begin{algorithmic}
\STATE {\bfseries Input:} Given $\rxz \sim q\left(\rxz\right)$, Inpainting mask $\vm$, DDPM trained weights $\vtheta$, Fixed variance $\beta > 0 $
\STATE Draw $\overleftarrow{\vx_1} \sim \gN\left(\0,\mI_d\right)$
\FOR{$r=1,\dots,R$}        
   \STATE Drift from reverse Markov process (\ref{eq:reverse}): $\mu_{\vtheta}\left(\overleftarrow{\vx_1}\right)$
   \STATE \textcolor{cyan}{Alignment of drift: $\overleftarrow{\vx_{0}} = \mu_{\vtheta}\left(\overleftarrow{\vx_1}\right) \times \frac{1}{\sqrt{1-\beta}}$}
   \STATE Inpainting: $\overleftarrow{\vx_{0}} = \vm \odot \overleftarrow{\vx_{0}} + \left(1-\vm \right)\odot\overrightarrow{\vx_{0}}$
   \STATE Update $\overleftarrow{\vx_1} \leftarrow \lxz$
\ENDFOR
\STATE \textcolor{cyan}{Drift from reverse Markov process (\ref{eq:reverse}): $\mu_{\vtheta}\left(\overleftarrow{\vx_0}\right)$}
\STATE \textcolor{cyan}{Alignment of drift: $\overleftarrow{\hat{\vx}_{0}} = \mu_{\vtheta}\left(\overleftarrow{\vx_0}\right)\times \frac{1}{\sqrt{1-\beta}}$}
\STATE {\bfseries Output:} $\overleftarrow{\hat{\vx}_0}$
\end{algorithmic}
\end{algorithm}

\section{Theoretical Results}
\label{sec:theory}
In this section, we present our main generative and inpainting results, first for a two-state model~(\wasyparagraph{\ref{subsec:gen-model}, \wasyparagraph{\ref{subsec:image-inpainting}, \wasyparagraph{\ref{subsec:image-inpainting-with-noise}}), and subsequently generalize to a multi-state diffusion model~(\wasyparagraph{\ref{subsec:main-slow-jump}}).

\subsection{Problem Setup}
\label{subsec:prob_setup}
The goal is to minimize the divergence  $\gD\left(p_\vtheta(\lxz), q(\rxz)\right)$, where $\gD(.,.)$ may be a total variation (TV) distance~\cite{chen2022sampling}, KL-divergence~\cite{ho2020denoising}, or Wasserstein distance~\cite{debortoli2022convergence}. For a generative model, the objective is met as long as the distribution of $\lxz$ matches with $q(\rxz)$. For diffusion based inpainting, however, it is necessary to generate an $\epsilon$-close image, i.e., $\left\|\lxz - \rxz \right\| \leq \epsilon$ and $\lxz$ lies in the support of $q(\rxz)$. Suppose the data distribution $q\left(\rxz\right)$ is supported on a $k$-dimensional subspace of $\R^d$. For $\mA\in \R^{d\times k}$, we have the original samples $\rxz = \mA\vz_0$, where $\vz_0 \in \R^k$ is distributed according to $\gN\left(\0,\mI_k\right)$ and $\mA$ is full rank, i.e., $rank(\mA) = k \leq d$. We ask if a model trained for a {\em generative modeling task} helps recover the missing parts in an {\em image inpainting task}. As noted earlier in \wasyparagraph{\ref{sec:intro}}, recovery under this model using optimization techniques is well understood. Our goal is to use this well established framework as a vehicle to analyze and provide new insights into sample recovery using DDPMs.


\begin{algorithm}[tb]
   \caption{RePaint$^+$ Inference General Case}
   \label{alg:repaint-inference-general}
\begin{algorithmic}
\STATE {\bfseries Input:} Given $\rxz \sim q\left(\rxz\right)$, DDPM trained weights $\vtheta$, Initial variance $\beta_0 = 0$, Variance schedule $\{\beta_t\}_{t\geq 0}$, Alignment coefficients for drift ($\{\omega_t\}_{t\geq 0}$) and dispersion ($\{\xi_t\}_{t\geq 0}$), Initial $\overleftarrow{\vx_T} \sim \gN\left(\0,\mI_d\right)$
\FOR{$t = T,\dots,1$}
    \FOR{$r=1,\dots,R$}
        \IF{$t>1$}
            \STATE $\overrightarrow{\vepsilon} \sim \gN\left( \0,\mI_d \right)$ and $\overleftarrow{\vepsilon} \sim \gN\left( \0,\mI_d \right)$ 
        \ELSE 
            \STATE  $\overrightarrow{\epsilon} = 0$ and $\overleftarrow{\epsilon}=0$
   \ENDIF
   \STATE Known part: $\overrightarrow{\vx_{t-1}} = \sqrt{\Bar{\alpha}_{t-1}} \rxz + \sqrt{1-\Bar{\alpha}_{t-1}} \overrightarrow{\vepsilon}$
   \STATE Drift $\mu_{\vtheta}\left(\overleftarrow{\vx_t},t\right) = \frac{1}{\sqrt{\alpha_t}}\left ( \overleftarrow{\vx_t} - \frac{\beta_t}{\sqrt{1-\Bar{\alpha}_t}} ~\vepsilon_\theta\left ( \overleftarrow{\vx_t}, t \right ) \right )$ and dispersion $\sqrt{\beta_t}\overleftarrow{\vepsilon}$ from reverse process (\ref{eq:reverse})
   \STATE \textcolor{cyan}{Alignment of drift: $\mu_{\vtheta}\left(\overleftarrow{\vx_t},t\right)\times \omega_t$   
   }
   \STATE \textcolor{cyan}{Alignment of dispersion: $ \sqrt{\beta_t}\overleftarrow{\vepsilon}\times \xi_t$}
   \STATE \textcolor{cyan}{Compute $\overleftarrow{\vx_{t-1}} = \mu_{\vtheta}\left(\overleftarrow{\vx_t},t\right)\times \omega_t + \sqrt{\beta_t}\overleftarrow{\vepsilon}\times \xi_t$}
   \STATE Inpainting: $\overleftarrow{\vx_{t-1}} = \vm \odot \overleftarrow{\vx_{t-1}} + \left(1-\vm \right)\odot\overrightarrow{\vx_{t-1}}$
   \IF{$r < R$ and $t>1$}
        \STATE Draw $\vepsilon \sim \gN\left(\0,\mI_d \right)$
        \STATE Pushforward: $\overleftarrow{\vx_t} = \sqrt{1-\beta_t} \overleftarrow{\vx_{t-1}} + \sqrt{\beta_t} \vepsilon$
   \ELSE
       \STATE Update $\overleftarrow{\vx_1} \leftarrow \lxz$
   \ENDIF
   \ENDFOR
\ENDFOR
\STATE Drift $\mu_{\vtheta}\left(\overleftarrow{\vx_1},1\right) = \frac{1}{\sqrt{\alpha_1}}\left ( \overleftarrow{\vx_1} - \frac{\beta_1}{\sqrt{1-\Bar{\alpha}_1}} ~\vepsilon_\vtheta\left ( \overleftarrow{\vx_1}, 1 \right ) \right )$ from reverse Markov process (\ref{eq:reverse})
   \STATE \textcolor{cyan}{Alignment of drift: $\mu_{\vtheta}\left(\overleftarrow{\vx_1},1\right)\times \omega_1$   
   }
   \STATE \textcolor{cyan}{Compute $\overleftarrow{\hat{\vx}_{0}} = \mu_{\vtheta}\left(\overleftarrow{\vx_1},1\right)\times \omega_1$}
\STATE {\bfseries Output:} $\overleftarrow{\hat{\vx}_0}$
\end{algorithmic}
\end{algorithm}

\subsection{Major Insights from Analysis}
\label{subsec:main-d-dim}
Our main result is that given the weights of a DDPM trained for generative modeling tasks, we can repurpose its objective to image inpainting without retraining for the inpainting task and show perfect recovery with linear rate of convergence. Under \textbf{Assumption~\ref{assm:ortho}}, we first show that the weights learned by DDPM training \textbf{Algorithm~\ref{alg:ddpm-train}} successfully capture the underlying data generating distribution, see \textbf{Theorem \ref{thm:gen-d-dim}}. Then, we prove in \textbf{Theorem~\ref{thm:inp-d-dim}} under \textbf{Assummption~\ref{assm:ortho} and \ref{assm:inpainting}} that the optimal solution of DDPM training \textbf{Algorithm~\ref{alg:ddpm-train}} generalizes well to image inpainting tasks with unknown masks, see \textbf{Algorithm~\ref{alg:repaint-inference-special}} and \textbf{Algorithm~\ref{alg:repaint-inference-general}}. Similar to the variance schedule $\{\beta_t\}_{t\geq0}$, the alignment coefficients for drift $\{\omega_t\}_{t\geq0}$ and dispersion $\{\xi_t\}_{t\geq0}$ can be computed beforehand. For instance, in \textbf{Algorithm~\ref{alg:repaint-inference-special}}, we choose $\beta_t=\beta$ and $\omega_t = 1/\sqrt{1-\beta}$\footnote{The proposed modifications over RePaint~\cite{lugmayr2022repaint} are \textcolor{cyan}{highlighted} in \textbf{Algorithm~\ref{alg:repaint-inference-special}} and \textbf{Algorithm~\ref{alg:repaint-inference-general}}.}.

\begin{assumption}
    \label{assm:ortho}
    The column vectors of data generating model $\mA$ are orthonormal, i.e., $\mA^T\mA = \mI_k$.
\end{assumption}

\begin{assumption}
    \label{assm:inpainting}
    Given an inpainting mask $\vm \in \{0,1\}^d$, the following holds true: $\mA\mA^T\mD(\vm) \prec  \mI_d$.
\end{assumption}

\textbf{Assumption~\ref{assm:ortho}} is a mild assumption used to simplify the expressions. 
For \textbf{Assumption~\ref{assm:inpainting}}, recall that $\mD(\vm)$ is a $d \times d$ diagonal matrix (a mask) with elements in diagonal entries set to $1$ whenever data is missing  (i.e., data is masked) and $0$ otherwise (data is available). The symbol $\prec$ indicates that the spectral norm of $\mA\mA^T\mD(\vm)$, denoted by $\lambda_{\max}$, is strictly less than 1.


\textbf{Assumption~\ref{assm:inpainting}} has a physical interpretation that agrees with the intuition. Since masking an image reduces its free energy, it is important that the mask $\mD(\vm)$ is well behaved so that sufficient energy is left in the masked image for a faithful reconstruction. If $\mD(\vm) = \mI$, then all the energy of the original signal is lost. In this scenario, it is impossible to recover the original signal. When $\mD(\vm)=\0$, it is a trivial case because the algorithm has access to the original signal itself. 
Our results on recoverability hold for all the remaining cases when $\vm \sim \{0,1 \}^d$ such that $\lambda_{\max} < 1$.

An interesting avenue for further research is to precisely characterize how much information we need in the form of $\mD(\vm)$ for perfect recovery. This problem has been extensively studied in compressed sensing literature~\cite{bora2017compressed,wu2019deep}, which is not the focus of this paper. We defer such questions to future work. 

\subsubsection{Generative Modeling using DDPM}
\label{subsec:gen-model}
Here, we present \textbf{Theorem~\ref{thm:gen-d-dim}} for computing the analytical solution of DDPM for a two-state model. In general, computing the analytical solution is a futile exercise as we do not know $q(\rxz)$ a priori. However, the additional structure on the data generating distribution, as typically assumed in downstream tasks such as image inpainting, allows us to derive an explicit form of the solution~\cite{bora2017compressed,wu2019deep}. We consider the objective function~(\ref{eq:vlbo_mu}) that estimates the mean of the posterior. One may wish to estimate the noise instead with the reparameterized objective~(\ref{eq:vlbo_reparam}). Both the choices produce similar experimental results~\cite{ho2020denoising}.

\begin{theorem}[Generative Modeling]
    Suppose \textbf{Assumption~\ref{assm:ortho}} holds. Let us denote $\vtheta^* =\arg_{\min} \gL\left(\vtheta\right)$, where $\gL\left(\vtheta \right)$ is defined as: $$ \E_{\rxz, \overrightarrow{\vepsilon}}\left[ \left \| \tilde{\mu}_1\left ( \overrightarrow{\vx_1}(\rxz, \overrightarrow{\vepsilon}), \overrightarrow{\vx_0} \right ) - \mu_\theta\left ( \overrightarrow{\vx_1}\left(\rxz, \overrightarrow{\vepsilon}\right) \right ) \right \|^2 \right].$$ For a fixed variance $\beta > 0$, if we consider a function approximator $\mu_\vtheta\left( \overrightarrow{\vx_1}\left(\rxz, \overrightarrow{\vepsilon}\right) \right ) = \vtheta \overrightarrow{\vx_1}\left(\rxz, \overrightarrow{\vepsilon}\right)$, then the closed-form solution $\vtheta^* = \sqrt{1-\beta}\mA\mA^T$, which upon renormalization by $\left(1/\sqrt{1-\beta}\right)$ recovers the true subspace of $q\left(\rxz\right)$.
    \label{thm:gen-d-dim}
\end{theorem}

\begin{proof}
    The proof is included in Appendix~\ref{subsec:prf-gen-d-dim}.
\end{proof}

 For any $\overleftarrow{\vx_1} \sim \gN\left(\0,\mI_d \right)$, the reverse process with the optimal solution $\vtheta^*$ generates $\overleftarrow{\vx_0} = \sqrt{1-\beta} \mA \mA^T \lxo$. After renormalization, it gives $\overleftarrow{\vx_0} = \mA \left( \mA^T \lxo\right)$ which is \textbf{unbiased} and has \textbf{identity covariance}, i.e., $\E\left[\mA^T\lxz \right]=\0$ and $\E\left[\left(\mA^T\lxo\right)\left(\mA^T\lxo\right)^T \right]=\E\left[\mA^T\left(\lxo\lxo^T\right)\mA \right]=\mI_k$. This verifies that \textbf{Algorithm~\ref{alg:ddpm-train}} recovers the subspace of the underlying data distribution, which is supported on a linear manifold with Gaussian marginals. In what follows, we provide a detailed analysis of the inpainting \textbf{Algorithm~\ref{alg:repaint-inference-special}}.

\subsubsection{Image Inpainting using RePaint\texorpdfstring{$^{+}$}{RePaintPlus}}
\label{subsec:image-inpainting}
Continuing with the two-state model, there are essentially three key ingredients to diffusion based image inpainting. First, we initialize the reverse process at $\lxo \sim \gN\left(\0,\mI_d\right)$. Using the reverse Gaussian transition kernel obtained in \textbf{Theorem~\ref{thm:gen-d-dim}}, we generate $\lxz = \left( 1/\sqrt{1-\beta}\right)\mu_{\vtheta^*}\left(\lxo\right)$. As per \textbf{Theorem~\ref{thm:gen-d-dim}}, $\lxz$ is \textit{a sample} that lies on the manifold. However, we note that $\lxz$ is \textit{not the sample} we are looking for as it could potentially lie far away from $\rxz$. 
    
    Second, we replace certain parts of $\lxz$ as per the given mask $\vm$ with known information from $\rxz$ using the following rule: $\overrightarrow{\vx_0^{1}} = \vm \odot \overleftarrow{\vx_0} + (\1-\vm) \odot \overrightarrow{\vx_0}$. This produces an arbitrary sample $\overrightarrow{\vx_0^1} \in \R^d$. Since the reverse SDE pushes an arbitrary sample in $\R^d$ onto the manifold, we feed $\overrightarrow{\vx_0^1}$ as an \textit{input} to the reverse SDE in the first resampling round. Alternatively, one may pass $\overrightarrow{\vx_0^1}$ through the forward SDE without adding any noise. In this case, the generated sample $\overrightarrow{\vx_1^1}$ serves as the input to the reverse SDE. This provides just an additional scaling that is easy to handle in our analysis.
    
    Finally, the reverse SDE generates $\overleftarrow{\vx_0^1}$, which upon realignment is expected to be $\vepsilon$-close to the original image, i.e., $\left\|\overleftarrow{\vx_0^1} - \rxz \right\| \leq \epsilon$. This completes one resampling step. 

    Usually, one resampling step is not sufficient to obtain an $\vepsilon$-close solution. We ask how many resampling steps \textbf{Algorithm~\ref{alg:repaint-inference-special}} needs to generate a satisfactory inpainted image. Furthermore, what is the rate at which it converges to the $\vepsilon$-neighborhood of the original image. We answer these questions favorably in \textbf{Theorem~\ref{thm:inp-d-dim}}.

\begin{theorem}[Image Inpainting]
    Let \textbf{Assumption~\ref{assm:ortho}} and \textbf{Assumption~\ref{assm:inpainting}} hold. Suppose $\lambda_{\max}\coloneqq \left\| \mA\mA^T\mD(\vm)\right\|$ and $\vtheta^* =\arg_{\min} \gL\left(\vtheta\right)$, where $\gL\left(\vtheta \right)$ is defined as: $$ \E_{\rxz, \overrightarrow{\vepsilon}}\left[ \left \| \tilde{\mu}_1\left ( \overrightarrow{\vx_1}(\rxz, \overrightarrow{\vepsilon}), \overrightarrow{\vx_0} \right ) - \mu_\theta\left ( \overrightarrow{\vx_1}\left(\rxz, \overrightarrow{\vepsilon}\right) \right ) \right \|^2 \right].$$ For any mask $\vm \sim \{0,1\}^d$, a fixed variance $\beta > 0$, a partially known image $\rxz \sim q\left(\rxz\right)$, and reverse SDE initialized at $\lxo \sim \gN\left(\0,\mI_d\right)$, the maximum number of resampling steps ($r$) needed by \textbf{Algorithm~\ref{alg:repaint-inference-special}} to recover an $\epsilon$-accurate inpainted image, i.e., $$\left\|\overleftarrow{\vx_0^r} - \rxz \right\| \leq \lambda_{\max}^r \biggl( \frac{\left\| \vtheta^*\right\| \left\| \lxo - \rxz \right\|}{\sqrt{1-\beta}}\biggr) \leq \epsilon$$
    is upper bounded by 
    \vspace{-0.2in}
    \begin{align*}
        \gO\left( \frac{\log\left( \frac{\left \| \theta^* \right \|\left \| \overleftarrow{x_1}- \rxz \right \|}{\epsilon\sqrt{1-\beta}}\right)}{\log\left(\frac{1}{\lambda_{\max}} \right)}\right).
    \end{align*}
    \label{thm:inp-d-dim}
\end{theorem}
\vspace{-0.4in}
\begin{proof}
    The proof is included in Appendix~\ref{subsec:prf-inp-d-dim}.
\end{proof}
\vspace{-0.1in}

We draw several key insights from \textbf{Theorem~\ref{thm:inp-d-dim}}.

\textbf{Universal mask prinicple:} An important observation is that the rate does not depend upon the mask $\vm$ as long as it is a valid inpainting mask. Thus, \textbf{Algorithm~\ref{alg:repaint-inference-special}} \textit{recovers the original sample} irrespective of whether it has been trained on such masked images or not. This is important because it allows us to repurpose the objective of a diffusion based generative model to address image inpainting. 

\textbf{Linear rate of convergence: } A major implication of \textbf{Theorem~\ref{thm:inp-d-dim}} is that \textbf{Algorithm~\ref{alg:repaint-inference-special}} enjoys a \textit{linear rate of convergence}. For this reason, we only need a small increase in the number of resampling steps to get a significant improvement in terms of $\epsilon$-accuracy. 

\textbf{Information bottleneck:} One interesting controlling parameter for the rate of convergence is $\lambda_{\max}$. From \textbf{Theorem~\ref{thm:inp-d-dim}}, it is evident that a large value of $\lambda_{\max}$ requires more resampling steps. Since $\lambda_{\max}$ correlates with missing information, it is understandable that \textbf{Algorithm~\ref{alg:repaint-inference-special}} needs \textit{more iterations for a reasonable harmonization}.

\textbf{Low norm solution:} Among other controlling parameters, \textbf{Algorithm~\ref{alg:repaint-inference-special}} seems to favor low norm solutions from generative modeling. This indicates that having an \textit{inductive bias} in DDPM training \textbf{Algorithm~\ref{alg:ddpm-train}} to prefer a low norm solution may assist in diffusion based inpainting. 

\textbf{Distance from initialization: } Finally, the distance from the original sample, i.e., $\left \| \overleftarrow{x_1}- \rxz \right \|$ increases the iteration complexity only \textit{logarithmically}.

\subsubsection{Image Inpainting with Noisy Generator}
\label{subsec:image-inpainting-with-noise}
An immediate consequence of \textbf{Theorem~\ref{thm:inp-d-dim}} follows when \textbf{Algorithm~\ref{alg:ddpm-train}} returns an \textit{approximate} solution $\hat{\vtheta} \coloneqq \vtheta^* + \mathbf{\delta}\sqrt{1-\beta}$ in contrast to the \textit{exact} analytical solution $\vtheta^*$. This is a reasonable setting because the closed-form solution is not exactly computable when the data generating distribution $q\left(\rxz\right)$ is not known a priori, which is the case for most interesting practical applications. We present \textbf{Theorem~\ref{thm:inp-d-dim-noise}} under \textbf{Assumption~\ref{ass:noisy-gen-d-dim}} taking into account the approximate error in generative modeling. 
\begin{assumption}
    The perturbation $\mdelta$ of the approximate solution is such that $\left(\mA\mA^T+\mdelta\right)\mD(\vm) \prec \mI_d$.
    \label{ass:noisy-gen-d-dim}
\end{assumption}

\begin{theorem}
    \label{thm:inp-d-dim-noise}
    Suppose \textbf{Assumption~\ref{ass:noisy-gen-d-dim}} holds and $\hat{\lambda}_{\max} \coloneqq \left\| \left(\mA\mA^T+\delta\right)\mD(\vm)\right\|$. Then, running \textbf{Algorithm~\ref{alg:repaint-inference-special}} for $r$ resampling steps with a $\delta$-approximate model $\hat{\vtheta}$ yields:
    \begin{align*}    
    \left\|\overleftarrow{\vx_0^r} - \rxz \right\| \leq \hat{\lambda}_{\max}^r \biggl(  \frac{\left\| \vtheta^*\right\| \left\| \lxo - \rxz \right\|}{\sqrt{1-\beta}}
    + \frac{\left\|\hat{\vtheta} - \vtheta^* \right\| \left \| \overrightarrow{\vx_0} \right \|}{\left ( 1-\hat{\lambda}_{\max} \right ) \sqrt{1-\beta}} \biggr)
    + \frac{\left\|\hat{\vtheta} - \vtheta^* \right\| \left \| \overrightarrow{\vx_0} \right \|}{\left ( 1-\hat{\lambda}_{\max} \right ) \sqrt{1-\beta}}.
    \end{align*}
\end{theorem}
\vspace{-0.2in}
\begin{proof}
    The proof is included in Appendix~\ref{subsec:prf-thm-inp-d-dim-noise}.
\end{proof}
\vspace{-0.1in}

\textbf{Theorem~\ref{thm:inp-d-dim-noise}} states that since $\hat{\lambda}_{\max}^{r} < 1$, we recover a $\zeta$-approximate solution in the limit, where $\zeta = \frac{\left\|\hat{\vtheta} - \vtheta^* \right\| \left \| \overrightarrow{\vx_0} \right \|}{\left ( 1-\hat{\lambda}_{\max}^r \right ) \sqrt{1-\beta}}$. The reconstruction error is proportional to the approximation error of DDPM in generative modeling. In the absence of noise, we have $\left\|\hat{\vtheta} - \vtheta^* \right\| \rightarrow 0$, which leads to recovery of the true underlying sample $\rxz$.

\begin{corollary}
    Suppose $q\left(\rxz\right)$ has a compact support with $\left\|\rxz\right\| \leq \kappa$. If $\delta=\gO\left(\epsilon  \left(1-\hat{\lambda}_{\max} \right)/\kappa\right)$, then 
    \begin{align*}
        \mathop{\sup}_{\rxz \sim q(\rxz)}\left\|\overleftarrow{\vx_0^r} - \rxz \right\| \leq \epsilon.
    \end{align*}
    \label{cor:thm-inp-d-dim-noise}
\end{corollary}
\vspace{-0.3in}
\begin{proof}
    The proof is included in Appendix~\ref{subsec:prf-cor-inp-d-dim}.
\end{proof}
\vspace{-0.1in}

In practice, most of the data generating distributions have a compact support. For these distributions, it is reasonable to have $\left\| \rxz \right\| \leq \kappa$. Then, we recover an $\epsilon$-accurate solution as long as the error of DDPM generator is of the order $\gO\left(\epsilon\right)$.

\subsubsection{Resampling vs Slowing Down Diffusion}
\label{subsec:main-slow-jump}
We now generalize to a {\em multi-state model}. In \wasyparagraph{\ref{subsec:image-inpainting}} and \wasyparagraph{\ref{subsec:image-inpainting-with-noise}}, we discussed that resampling plays a vital role in harmonizing the missing parts of an image. We ask whether a similar result is achievable by slowing down the diffusion process over $(T+1)$ states without resampling, instead of $2$ states with many resampling steps. To this end, we first derive a closed-form solution for generative modeling with $(T+1)$ diffusion states in \textbf{Theorem~\ref{thm:gen-d-dim-t-states}}. Then, we provide a simplified closed-form solution in \textbf{Corollary~\ref{cor:gen-d-dim-t-states}} assuming that the noise in the forward SDE and the reverse SDE are Independent and Identically Distributed (IID). Using this solution, we show in \textbf{Theorem~\ref{thm:inp-d-dim-t-states}} that inpainting without resampling at intermediate states yields inferior results compared to less states with many resampling steps. 

\begin{theorem}
\label{thm:gen-d-dim-t-states}
Suppose \textbf{Assumption~\ref{assm:ortho}} holds. For a deterministic variance schedule $\{\beta_t\}_{t\geq0}$, let $\alpha_t = 1 - \beta_t$, $\Bar{\alpha}_t = \prod_{s=0}^{t} \alpha_s$, 
    $
    \gamma 
    \coloneqq
    \mathop{\mathbb{E}}_{\substack{t}}
                \left(
                    \frac{1}{\left(1-\Bar{\alpha}_t \right)}
                \right) 
        \mathop{\mathbb{E}}_{\substack{t}}
         \left(
            \sqrt{\alpha}_t\left ( 1-\Bar{\alpha}_{t-1} \right )
         \right)$
         and 
         $\nu 
    \coloneqq 
        \mathop{\mathbb{E}}_{\substack{t}}
                \left(
                    \frac{1}{\left(1-\Bar{\alpha}_t \right)}
                \right) 
        \mathop{\mathbb{E}}_{\substack{t}}
         \left(
         \Bar{\alpha}_{t-1} \sqrt{\alpha_t} 
         \right)
         -
         \mathop{\mathbb{E}}_{\substack{t}}
                \left(
                    \frac{\Bar{\alpha}_t}{\left(1-\Bar{\alpha}_t \right)}
                \right)
        \mathop{\mathbb{E}}_{\substack{t}}
         \left(
            \sqrt{\alpha}_t
         \right)
    $,
         where $t\sim Uniform\{1,\cdots,T \}$. Let us denote $\vtheta^* =\arg_{\min} \gL\left(\vtheta\right)$, where $\gL\left(\vtheta \right)$ is defined as: 
$$ \mathop{\mathbb{E}}_{\substack{\rxz, \vepsilon, t}}\left [ \left \| \tilde{\mu}_t\left ( \overrightarrow{\vx_t}, \overrightarrow{\vx_0} \right ) - \mu_\theta\left ( \overrightarrow{\vx_t}, t \right ) \right \|^2
\right ].$$ For diffusion over $(T+1)$ states,  if $\mu_\theta\left ( \overleftarrow{\vx_t}, t \right ) = \vtheta \begin{bmatrix}
    \overleftarrow{\vx_t} \\ t
\end{bmatrix}$, where $\vtheta \in \R^{d\times (d+1)}$, then the closed-form solution of DDPM \textbf{Algorithm~\ref{alg:ddpm-train}} is given by $\vtheta^* = \left[\nu \mA\mA^T + \gamma \mI_d, \0\right]$.
\end{theorem}
\begin{proof}
    The proof is included in Appendix~\ref{subsec:prf-thm-gen-d-dim-t-state}.
\end{proof}

In general, $\{\beta_t\}_{t\geq 0 }$ captures a list of variances ranging from $\beta_{\min}$ to $\beta_{\max}$, such as a cosine or a linear schedule~\cite{ho2020denoising}. To capture the general framework, we provide $\nu$ and $\gamma$ as explicit functions of $\beta_t$. One may wish to compute these values exactly depending on the choice of variance schedule. For instance, in case of diffusion over 2 states with $\beta_0=0$ and $\beta_1 = \beta$, we have $\nu = 
                \left(
                    \frac{1}{\left(1-\Bar{\alpha}_1 \right)}
                \right) 
         \left(
         \Bar{\alpha}_{0} \sqrt{\alpha_1} 
         \right)
         -
                \left(
                    \frac{\Bar{\alpha}_1}{\left(1-\Bar{\alpha}_1 \right)}
                \right)
         \left(
            \sqrt{\alpha}_1
         \right) = \sqrt{1-\beta}$ 
         and
         $\gamma =
                \left(
                    \frac{1}{\left(1-\Bar{\alpha}_1 \right)}
                \right) 
         \left(
            \sqrt{\alpha}_t\left ( 1-\Bar{\alpha}_{0 } \right )
         \right) = 0$. This matches with the closed-form solution derived in \textbf{Theorem~\ref{thm:gen-d-dim}}.
         
         For the sake of analysis, we present a simplified closed-form solution in \textbf{Corollary~\ref{cor:gen-d-dim-t-states}} under \textbf{Assumption~\ref{assm:indepedence}}.

\begin{assumption}
    Let us denote by $\overrightarrow{\vepsilon}$ the noise used in the dispersion of the forward SDE, i.e., $\overrightarrow{\vx_{t-1}} = \sqrt{\Bar{\alpha}_{t-1}} \rxz + \sqrt{1-\Bar{\alpha}_{t-1}} \overrightarrow{\vepsilon}$, and $\overleftarrow{\vepsilon}$ in the reverse SDE, i.e., $\overleftarrow{\vx_{t-1}} = \mu_{\vtheta} \left(\overleftarrow{\vx_t}, t\right) + \sqrt{\beta_t} \overleftarrow{\vepsilon}$, where $\overleftarrow{\vx_t} = \sqrt{\Bar{\alpha}_t} \rxz + \sqrt{1-\Bar{\alpha}_t} \overleftarrow{\vepsilon}$. Let $\overrightarrow{\vepsilon}$ and $\overleftarrow{\vepsilon}$ be IID Gaussian random variables drawn from $\gN\left(\0,\mI_d\right)$.
    \label{assm:indepedence}
\end{assumption}

\begin{corollary}
    \label{cor:gen-d-dim-t-states}
    Suppose \textbf{Assumption~\ref{assm:indepedence}} holds.  For $\Bar{\nu} \coloneqq
        \mathop{\mathbb{E}}_{\substack{t}}
                \left(
                    \frac{1}{\left(1-\Bar{\alpha}_t \right)}
                \right) 
        \mathop{\mathbb{E}}_{\substack{t}}
         \left(
         \Bar{\alpha}_{t-1} \sqrt{\alpha_t} 
         \right)
         -
                \mathop{\mathbb{E}}_{\substack{t}}
                \left(
                    \frac{\Bar{\alpha}_t}{\left(1-\Bar{\alpha}_t \right)}
                \right)
                \mathop{\mathbb{E}}_{\substack{t}}
             \left(
             \Bar{\alpha}_{t-1} \sqrt{\alpha_t} 
             \right)
         $ in the same setting as \textbf{Theorem~\ref{thm:gen-d-dim-t-states}}, the closed-form solution of DDPM \textbf{Algorithm~\ref{alg:ddpm-train}} becomes $\vtheta^* = \left[ \Bar{\nu} \mA\mA^T, \0 \right]$.
\end{corollary}
\begin{proof}
    The proof is included in Appendix~\ref{subsec:prf-cor-gen-d-dim-t-state}.
\end{proof}
For simplicity, let us use this closed-form solution for inpainting at intermediate states. We prove in \textbf{Theorem~\ref{thm:gen-d-dim-t-states-sample}} that for diffusion over $(T+1)$ states, the generative model using DDPM learns to sample from the data manifold. 


\begin{theorem}
    \label{thm:gen-d-dim-t-states-sample}
    For the coefficients of drift $\omega_t = \omega = \frac{1}{\bnu\sqrt{2}^{1/T}}$ and dispersion $\xi_t = \sqrt{\frac{2^{(t-1)/T}}{2\beta_t(T-1)}}$, if we choose the Gaussian transition kernel using $\vtheta^* = \left[ \Bar{\nu} \mA\mA^T, \0 \right]$ in the reverse Markov process (\ref{eq:reverse}), then 
    \begin{eqnarray*}
        TV\left(p_{\vtheta^*}\left( \lxz\right), q\left(\rxz \right) \right) = 0.
    \end{eqnarray*}
\end{theorem}
\begin{proof}
    The proof is included in Appendix~\ref{subsec:prf-thm-gen-d-dim-t-states-sample}.    
\end{proof}

Since the total variation distance is zero, it ensures that DDPM learns the target distribution $q\left(\rxz \right)$. Now, let us see whether inpainting known portions of $\rxz$ at intermediate states would recover the true sample. In this regard, an important distinction from RePaint$^+$ is that there is no resampling at intermediate states.

\begin{theorem}
    \label{thm:inp-d-dim-t-states}
    For the coefficients of drift $\omega_t = \omega = \frac{1}{\bnu\sqrt{2}^{1/T}}$ and dispersion $\xi_t = \sqrt{\frac{2^{(t-1)/T}}{2\beta_t(T-1)}}$, following the reverse Markov process (\ref{eq:reverse}) with inpainting at intermediate states, generates a sample satisfying:
    \begin{align*}
        \nonumber
        \lxz 
        & = \frac{1}{\sqrt{2}} \left(\mA\mA^T \mD(\vm) \right)^{T-1} \mA\mA^T\overleftarrow{\vx_T}
        \\
        \nonumber
        & + 
        \left(\sum_{t=1}^{T-1} \left( \frac{1}{\sqrt{2}}\right)^{t/T} \sqrt{\Bar{\alpha}_t}  \left(\mA\mA^T \mD(\vm) \right)^{t-1} \right)
        \mA\mA^T \mD(\1-\vm) \rxz
        + \overrightarrow{\Bar{\Sigma}} \overrightarrow{\vepsilon}
        + \overleftarrow{\Bar{\Sigma}} \overleftarrow{\vepsilon},
    \end{align*}
    where $\overrightarrow{\Bar{\Sigma}}$ and $\overleftarrow{\Bar{\Sigma}}$ denote the covariance of collective dispersion in the forward and the reverse processes, respectively. 
\end{theorem}
\begin{proof}
    The proof is included in Appendix~\ref{subsec:thm-inp-d-dim-t-states}.
\end{proof}

We note that $\lxz$ is not necessarily an image residing on the manifold of the data generative distribution $q\left(\rxz \right)$. This is partly due to the residual bias arising from the noise injected in the forward and the reverse processes. Another cause of concern is the lack of harmonization between generated portions and rest of the image. Notably, the bias is persistent even if we consider diffusion over infinitely many states, justifying the importance of resampling over slowing down the diffusion process. This phenomenon is also observed in practice (\wasyparagraph{5},~\cite{lugmayr2022repaint}). 

On the contrary, diffusion over two states with many resampling steps recovers the original sample as discussed in \wasyparagraph{\ref{subsec:image-inpainting}}. In \textbf{Algorithm~\ref{alg:repaint-inference-general}}, we generalize this notion to $(T+1)$ states, where the generated image is harmonized at each intermediate state. Thus, we gradually wipe out the contribution of noise before it gets accumulated through the reverse transitions, leading to a smooth recovery.

\section{Conclusion}
\label{sec:conc}
We studied mathematical underpinnings of sample recovery in image inpainting using denoising diffusion probabilistic models. We presented a universal mask principle to prove that diffusion based inpainting  easily adapts to unseen masks without retraining. Our theoretical analysis of a popular inpainting algorithm in a linear setting
revealed a previously not understood bias that hampered perfect recovery. With realignment of drift and dispersion in the reverse process, we proposed an algorithm called Repaint$^+$ that eliminated this bias, leading to perfect recovery in the limit. Further, We proved that RePaint$^+$ converged linearly to the $\epsilon$-neighborhood of the true underlying sample. 
In the future, it would be interesting to study the convergence properties in a more general non-linear manifold setting.

\newpage
\printbibliography

\newpage
\appendix

\section{Technical Proofs}
\label{sec:proofs}



\subsection{Proof of Theorem~\ref{thm:gen-d-dim}}
\label{subsec:prf-gen-d-dim}
   For diffusion over two states, the posterior mean $\tilde{\mu}_t\left ( \overrightarrow{\vx_t},\overrightarrow{\vx_0} \right ) \coloneqq \frac{\sqrt{\Bar{\alpha}_{t-1} }\beta_t}{1-\Bar{a}_t}\overrightarrow{\vx_0} + \frac{\sqrt{\alpha}_t\left ( 1-\Bar{\alpha}_{t-1} \right )}{1-\Bar{\alpha}_t} \overrightarrow{\vx_t }$ simplifies to $\tilde{\mu}_1\left ( \overrightarrow{\vx_1},\overrightarrow{\vx_0} \right ) = \rxz$. Thus, the training loss becomes: 
\begin{eqnarray}
\nonumber
     \min_{\vtheta} \E_{\rxz, \overrightarrow{\vepsilon}}\left[ \left \| \tilde{\mu}_1( \overrightarrow{\vx_1}\left(\rxz, \overrightarrow{\vepsilon}), \overrightarrow{\vx_0} \right ) - \mu_\theta\left ( \overrightarrow{\vx_1}\left(\rxz, \overrightarrow{\vepsilon}\right) \right ) \right \|^2 \right] 
     =
     \\
     \nonumber
     \E_{\rxz, \overrightarrow{\vepsilon}}\left[\left\|\overrightarrow{\vx_0} - \mu_\vtheta\left(\overrightarrow{\vx_1}\right) \right\|^2\right] 
     = \E_{\rxz, \overrightarrow{\vepsilon}}\left [ \left \| \overrightarrow{\vx_0} - \vtheta \overrightarrow{\vx_1} \right \|^2 \right ] 
     =
     \\
     \nonumber
      \mathop{\mathbb{E}}_{\substack{\rxz \sim q\\ \reps \sim \mathcal{N}\left (\0,\mI_d \right )}} \left [ \left \| \rxz - \vtheta \left(\vx_0 \sqrt{1-\beta} + \reps \sqrt{\beta}\right) \right \|^2 \right ]
      =
      \\
    \mathop{\mathbb{E}}_{\substack{\rxz \sim q\\ \reps \sim \mathcal{N}\left (\0,\mI_d \right )}} \left [ \sum_{i=1}^{d} \left ( \overrightarrow{\vx_{0,i}} - \vtheta_i^T\left ( \rxz \sqrt{1-\beta} + \reps\sqrt{\beta}  \right ) \right )^2 \right ],
    \label{eq:lem-obj-d-dim}
\end{eqnarray}
where $\vtheta_i^T$ denotes the $i^{th}$ row of $\vtheta$. Solving the MMSE problem (\ref{eq:lem-obj-d-dim}) for $\vtheta_i^*$ yields
\begin{eqnarray}
    \nonumber 
    \vtheta_i^*  = \mathop{\mathbb{E}}_{\substack{\rxz, \reps}} \left [ \left ( \rxz \sqrt{1-\beta} + \reps \sqrt{\beta}\right )\left ( \rxz \sqrt{1-\beta} + \reps \sqrt{\beta}\right )^T \right ]^{-1} \mathbb{E}_{\rxz, \reps}\left [ \overrightarrow{\vx_{0,i}}\left ( \rxz \sqrt{1-\beta} + \reps \sqrt{\beta} \right ) \right ]
 =
 \\
 \nonumber
    \mathop{\mathbb{E}}_{\substack{\rzz, \reps}} \left [ \left ( \mathcal{A}\vz_0 \sqrt{1-\beta} + \reps \sqrt{\beta}\right )\left ( \mathcal{A} \vzz \sqrt{1-\beta} + \reps \sqrt{\beta}\right )^T \right ]^{-1}  \mathbb{E}_{\vzz, \reps}\left [ \left (\va_i^T \vzz   \right )\left ( \mathcal{A}\vzz \sqrt{1-\beta} + \reps \sqrt{\beta} \right ) \right ] 
=
\\
\label{eq:lem-obj-d-dim-1}
 \Big[(1-\beta)\mathcal{A}\mathbb{E}_{\vzz} \left[\vzz \vzz^T \right] \mathcal{A}^T + \beta \mathbb{E}_{\reps} \left[\reps \reps^T \right] \Big]^{-1} 
\left [\left(\sqrt{1-\beta} \right) \mathcal{A}\mathbb{E}_{\rxz}\left[\rxz \rxz^T \right] \va_i  \right ]. 
\end{eqnarray}
Using the fact that $\vz_0$ and $\vepsilon$ are independent Gaussian random variables with unit covariance, (\ref{eq:lem-obj-d-dim-1}) gives 
\begin{eqnarray}
    \nonumber 
    \vtheta_i^* = \left [(1-\beta)\mathcal{A} \mathcal{A}^T  + \beta\mI_d \right ]^{-1} 
\left [\left(\sqrt{1-\beta} \right) \mathcal{A} \va_i  \right ] 
=
\\
\nonumber
 \frac{\sqrt{1-\beta} }{\beta} \left [\mI_d+\left(\sqrt{\frac{1-\beta}{\beta}}\mathcal{A}\right) \left(\sqrt{\frac{1-\beta}{\beta}}\mathcal{A}\right)^T \right ]^{-1} \mathcal{A} \va_i. 
\end{eqnarray}
Next, we use Woodbury matrix identity~\cite{hager1989updating} to get 
 \begin{eqnarray}
     \nonumber 
    \vtheta_i^* 
    =
    \\
    \nonumber
     \frac{\sqrt{1-\beta} }{\beta} \Big[\mI_d -  \left(\sqrt{\frac{1-\beta}{\beta}}\mathcal{A}\right) \biggl(\mI_k + \left(\sqrt{\frac{1-\beta}{\beta}}\mathcal{A}\right)^T  \left(\sqrt{\frac{1-\beta}{\beta}}\mathcal{A}\right)\biggr)^{-1} \left(\sqrt{\frac{1-\beta}{\beta}}\mathcal{A}\right)^T   \Big]\mathcal{A} \va_i
    =
    \\
    \nonumber
 \frac{\sqrt{1-\beta} }{\beta} \Big[\mI_d -  \left(\sqrt{\frac{1-\beta}{\beta}}\mA\right)  \biggl(\mI_k\left(\frac{1-\beta}{\beta}\right)\mathcal{A}^T\mathcal{A}\biggr)^{-1} \left(\sqrt{\frac{1-\beta}{\beta}}\mathcal{A}\right)^T   \Big]\mathcal{A} \va_i
 =
 \\
 \nonumber
 \frac{\sqrt{1-\beta} }{\beta} \left[\mI_d -  \left(\sqrt{\frac{1-\beta}{\beta}}\mA\right) \beta\mI_k \left(\sqrt{\frac{1-\beta}{\beta}}\mathcal{A}\right)^T   \right]\mathcal{A} \va_i 
=
\\
\nonumber
 \frac{\sqrt{1-\beta} }{\beta} \left[\mI_d -  \left(1-\beta\right)\mA \mA^T \right]\mathcal{A} \va_i 
 =
 \\
 \nonumber
 \frac{\sqrt{1-\beta} }{\beta} \left[ \mathcal{A} \va_i -  \left(1-\beta\right)\mA\va_i \right] = \sqrt{1-\beta} \mA \va_i.
    \label{eq:lem-obj-d-dim-2}
\end{eqnarray} 
Stacking all the rows together, we get the optimal solution of the MMSE problem (\ref{eq:lem-obj-d-dim}), i.e., $\vtheta^* =  \sqrt{1-\beta} \mA \mA^T$. This optimal solution characterizes the Gaussian transition kernel of the reverse SDE (\ref{eq:reverse}).  \hfill $\square$

\subsection{Proof of Theorem~\ref{thm:inp-d-dim}}
\label{subsec:prf-inp-d-dim}
    After $r$ resampling steps, we get the following output from the reverse process:
    \begin{eqnarray}
        \nonumber
        \overleftarrow{\vx_0^r} 
        =
        \\
        \nonumber
        \mA \mA^T \left( \mD \left( \vm\right) \overleftarrow{\vx_0^{r-1}} + \mD \left( \1-\vm\right) \overrightarrow{\vx_0}\right)
        =
        \\
        \nonumber
        \mA \mA^T \mD \left( \vm\right) \overleftarrow{\vx_0^{r-1}} + \mA \mA^T \mD \left( \1-\vm\right) \overrightarrow{\vx_0} 
        =
        \\
        \nonumber
         \mA \mA^T \mD \left( \vm\right) \biggl(\mA \mA^T \mD \left( \vm\right) \overleftarrow{\vx_0^{r-2}} 
                 + \mA \mA^T \mD \left( \1-\vm\right) \overrightarrow{\vx_0} \biggr) + \mA \mA^T \mD \left( \1 - \vm\right) \overrightarrow{\vx_0}     
        =
        \\
         \left( \mA \mA^T \mD \left( \vm\right) \right)^r \lxz + \left( \mI+\sum_{i=1}^{r-1} \left(\mA\mA^T \mD(\vm) \right)^i\right)
         \biggl(\mA\mA^T \mD(\1-\vm)\mA\vz_0\biggr).
        \label{eq:thm-inp-d-dim-1}
    \end{eqnarray}
     Using the truncation of Neumann series, (\ref{eq:thm-inp-d-dim-1}) simplifies to  
    \begin{eqnarray}
        \nonumber
        \overleftarrow{\vx_0^r} = \left( \mA \mA^T \mD \left( \vm\right) \right)^r \lxz + \left(\mI - \left(\mA\mA^T\mD(\vm) \right)^r \right)
         \left(\mI - \mA\mA^T\mD(\vm) \right)^{-1}\mA\mA^T \mD(\1-\vm)\mA\vz_0 
         =
         \\
        \nonumber
         \left( \mA \mA^T \mD \left( \vm\right) \right)^r \lxz + \left(\mI - \left(\mA\mA^T\mD(\vm) \right)^r \right)  
         \left(\mI - \mA\mA^T\mD(\vm) \right)^{-1}\mA\mA^T \left(\mI - \mD(\vm) \right)\mA\vz_0
         =
         \\
        \nonumber
         \left( \mA \mA^T \mD \left( \vm\right) \right)^r \lxz + \left(\mI - \left(\mA\mA^T\mD(\vm) \right)^r \right) 
        \nonumber
         \left(\mI - \mA\mA^T\mD(\vm) \right)^{-1}\left(\mA\mA^T\mA\vz_0  - \mA\mA^T \mD(\vm)\mA\vz_0 \right) 
         =
         \\
        \nonumber
         \left( \mA \mA^T \mD \left( \vm\right) \right)^r \lxz + \left(\mI - \left(\mA\mA^T\mD(\vm) \right)^r \right) 
        \nonumber
         \left(\mI - \mA\mA^T\mD(\vm) \right)^{-1}\left(\mA\vz_0  - \mA\mA^T \mD(\vm)\mA\vz_0 \right)
         =
         \\
        \nonumber
         \left( \mA \mA^T \mD \left( \vm\right) \right)^r \mA\mA^T \lxo + \left(\mI - \left(\mA\mA^T\mD(\vm) \right)^r \right) 
        \left(\mI - \mA\mA^T\mD(\vm) \right)^{-1}\left(\mI  - \mA\mA^T \mD(\vm)\right)\mA\vz_0 
        =
        \\
         \left( \mA \mA^T \mD \left( \vm\right) \right)^r \left( \mA\mA^T \lxo - \mA\vz_0\right) + \mA\vz_0. 
        \label{eq:thm-inp-d-dim-2}
    \end{eqnarray}
    Next, we use the Singular Value Decomposition (SVD) of $\mA\mA^T \mD(\vm)$. Let us denote the left singular vectors of $\mA\mA^T \mD(\vm)$ by $\mU \in \R^{d\times d}$, the eigen values by the diagonal matrix $\mathbf{\Sigma} \in \R^{d\times d}$, and the right singular vectors by $\mV \in \R^{d \times d}$. Therefore, we simplify the expression (\ref{eq:thm-inp-d-dim-2}) to 
    \begin{align*}
        \overleftarrow{\vx_0^r} &= \mU \mathbf{\Sigma}^r \mV  \left( \mA\mA^T \lxo - \mA\vz_0\right) + \mA\vz_0
        \label{eq:thm-inp-d-dim-3}.
    \end{align*}
    Since the largest eigen value of $\mA\mA^T\mD(\vm)$ is stricly less than $1$, we get the following result: $\lim_{r\rightarrow\infty} \overleftarrow{\vx_0^r} = \mA\vz_0$. This answers the question of recoverability by diffusion based inpainting. Indeed, \textbf{Algorithm~\ref{alg:repaint-inference-special}} recovers the \textit{exact sample} given infinitely many resampling steps. For most practical applications, it is sufficient to get an $\vepsilon$-close solution. Now, let us derive the rate at which \textbf{Algorithm~\ref{alg:repaint-inference-special}} converges to the $\vepsilon$-neighborhood of $\mA\vz_0$.
    
    Subtracting the original sample $\rxz = \mA\vz_0$ from both sides of (\ref{eq:thm-inp-d-dim-2}) and taking the Euclidean norm, we arrive at 
    \begin{eqnarray}
    \nonumber
        \left\|\overleftarrow{\vx_0^r} - \rxz \right\|  =
        \\
        \nonumber
        \left\|  \left( \mA \mA^T \mD \left( \vm\right) \right)^r \left( \mA\mA^T \lxo - \mA\vz_0\right) \right\| 
        \stackrel{(i)}{\leq}
        \\
        \nonumber
          \left\| \mA\mA^T \mD(\vm) \right\|^r \left\|  \mA\mA^T \lxo - \mA\vz_0\right\|
          =
          \\
        \nonumber
         \left\| \mA\mA^T \mD(\vm) \right\|^r \left\|  \mA\mA^T \lxo - \mA\mA^T\mA\vz_0\right\|
         \leq 
         \\
        \nonumber
         \lambda_{\max} ^r \frac{\left\| \vtheta^*\right\|}{\sqrt{1-\beta}} \left\| \lxo - \rxz \right\|,
    \end{eqnarray}
    where (i) uses Cauchy-Schwarz inequality and the spectral norm identity. By making $\lambda_{\max} ^r \frac{\left\| \vtheta^*\right\|}{\sqrt{1-\beta}} \left\| \lxo - \rxz \right\|  \leq \epsilon$, we ensure that the algorithm converges to an $\epsilon$-accurate solution. We achieve this convergence by running at least $\frac{\log\left( \frac{\left \| \theta^* \right \|\left \| \overleftarrow{x_1}- \rxz \right \|}{\epsilon\sqrt{1-\beta}}\right)}{\log\left(\frac{1}{\lambda_{\max}} \right)}$ resampling steps, which implies a \textit{linear rate} of convergence.  \hfill $\square$

\subsection{Proof of Theorem~\ref{thm:inp-d-dim-noise}}
\label{subsec:prf-thm-inp-d-dim-noise}
    Similar to the proof of \textbf{Theorem~\ref{thm:inp-d-dim}}, we begin with the following expression in the case of noisy generative model: 
    \begin{eqnarray}
    \nonumber
        \overleftarrow{\vx_0^r} = \left(\mA \mA^T+\mdelta \right)\left( \mD \left( \vm\right) \overleftarrow{\vx_0^{r-1}} + \mD \left( \1-\vm\right) \overrightarrow{\vx_0}\right) 
        =
        \\
        \nonumber
        \left(\mA \mA^T+\mdelta \right)\mD \left( \vm\right) \overleftarrow{\vx_0^{r-1}} + \left(\mA \mA^T+\mdelta \right) \mD \left( \1-\vm\right) \overrightarrow{\vx_0} 
        =
        \\
        \nonumber
        \left(\mA \mA^T+\mdelta \right) \mD \left( \vm\right) \biggl(\left(\mA \mA^T+\mdelta \right)\mD \left( \vm\right) \overleftarrow{\vx_0^{r-2}} 
                 + \left(\mA \mA^T+\mdelta \right) \mD \left( \1-\vm\right) \overrightarrow{\vx_0} \biggr) + \left(\mA \mA^T+\mdelta \right) \mD \left( \1 - \vm\right) \overrightarrow{\vx_0}     
         =
         \\
        \nonumber
         \left( \left(\mA \mA^T+\mdelta \right) \mD \left( \vm\right) \right)^r \lxz + \left( \sum_{i=0}^{r-1} \left(\left(\mA \mA^T+\mdelta \right) \mD(\vm) \right)^i\right)
         \biggl(\left(\mA \mA^T+\mdelta \right) \mD(\1-\vm)\mA\vz_0\biggr) 
         =
         \\
        \nonumber
        \left( \left(\mA \mA^T+\mdelta \right) \mD \left( \vm\right) \right)^r \lxz + \left(\mI - \left(\left( \mA \mA^T+\mdelta\right)\mD(\vm) \right)^r \right) \left(\mI - \left( \mA \mA^T+\mdelta\right)\mD(\vm) \right)^{-1} \\
        \nonumber
        \times \left(\mA \mA^T\mA\vz_0 + \mdelta \mA\vz_0  - \left( \mA \mA^T+\mdelta\right) \mD(\vm)\mA\vz_0 \right)
        =
        \\
        \nonumber
        \left( \left( \mA \mA^T+\mdelta\right) \mD \left( \vm\right) \right)^r \lxz + \left(\mI - \left(\left( \mA \mA^T+\mdelta\right)\mD(\vm) \right)^r \right) \left(\mI - \left( \mA \mA^T+\mdelta\right)\mD(\vm) \right)^{-1}
        \\
        \nonumber
        \times \left(\mA\vz_0   - \left( \mA \mA^T+\mdelta\right) \mD(\vm)\mA\vz_0 + \mdelta \mA\vz_0\right)
        =
        \\
        \nonumber 
        \left( \left( \mA \mA^T+\mdelta\right) \mD \left( \vm\right) \right)^r \lxz + \left(\mI - \left(\left( \mA \mA^T+\mdelta\right)\mD(\vm) \right)^r \right) \left(\mI - \left( \mA \mA^T+\mdelta\right)\mD(\vm) \right)^{-1}
        \\
        \nonumber
        \times \left(\left(\mI   - \left( \mA \mA^T+\mdelta\right) \mD(\vm)\right)\mA\vz_0 + \mdelta \mA\vz_0\right)
        = 
        \\
        \nonumber 
        \left( \left( \mA \mA^T+\mdelta\right) \mD \left( \vm\right) \right)^r \lxz + \left(\mI - \left(\left( \mA \mA^T+\mdelta\right)\mD(\vm) \right)^r \right)
        \nonumber
         \left(\mA\vz_0 + \left(\mI - \left( \mA \mA^T+\mdelta\right)\mD(\vm) \right)^{-1}\mdelta \mA\vz_0\right)
         =
         \\
        \nonumber
         \left( \left( \mA \mA^T+\mdelta\right) \mD \left( \vm\right) \right)^r \lxz + \mA\vz_0 + \left(\mI - \left( \mA \mA^T+\mdelta\right)\mD(\vm) \right)^{-1}\mdelta \mA\vz_0 - \left(\left( \mA \mA^T+\mdelta\right)\mD(\vm) \right)^r \mA\vz_0\\
        \nonumber
         - \left(\left( \mA \mA^T+\mdelta\right)\mD(\vm) \right)^r \left(\mI - \left( \mA \mA^T+\mdelta\right)\mD(\vm) \right)^{-1}\mdelta \mA\vz_0. 
         \label{eq:prf-thm-inp-d-dim-noise}
    \end{eqnarray}
    Subtracting $\rxz = \mA\vz_0$ from both sides of the above expression 
    and taking the Euclidean norm, we get 
    \begin{eqnarray}
    \nonumber
        \left\|\overleftarrow{\vx_0^r} - \rxz\right\| = \Biggl\| \left( \left( \mA \mA^T+\mdelta\right) \mD \left( \vm\right) \right)^r \lxz + \left(\mI - \left( \mA \mA^T+\mdelta\right)\mD(\vm) \right)^{-1}\mdelta \mA\vz_0 - \left(\left( \mA \mA^T+\mdelta\right)\mD(\vm) \right)^r \mA\vz_0 \\
        \nonumber
        - \left(\left( \mA \mA^T+\mdelta\right)\mD(\vm) \right)^r \left(\mI - \left( \mA \mA^T+\mdelta\right)\mD(\vm) \right)^{-1}\mdelta \mA\vz_0
        \Biggr\|.
        \label{eq:prf-thm-inp-d-dim-noise-1}
    \end{eqnarray}
    After rearranging the terms, 
    we apply the triangle inequality followed by Cauchy-Schwarz inequality to obtain
    \begin{eqnarray}
    \nonumber
        \left\|\overleftarrow{\vx_0^r} - \rxz\right\| 
        \leq
        \\
        \nonumber
        \left\| \left( \left( \mA \mA^T+\mdelta\right) \mD \left( \vm\right) \right)^r \right\| 
        \left\|\left( \lxz - \mA\vz_0 -  \left(\mI - \left( \mA \mA^T+\mdelta\right)\mD(\vm) \right)^{-1}\mdelta \mA\vz_0
        \right) \right\| \\
        \nonumber
        + \left\| \left(\mI - \left( \mA \mA^T+\mdelta\right)\mD(\vm) \right)^{-1}\mdelta \mA\vz_0  \right\| 
        \leq 
        \\
        \nonumber
        \left\|  \left( \mA \mA^T+\mdelta\right) \mD \left( \vm\right) \right\|^r 
        \left(\left\| \lxz - \mA\vz_0\right\|+ \left\| \left(\mI - \left( \mA \mA^T+\mdelta\right)\mD(\vm) \right)^{-1}\mdelta \mA\vz_0
         \right\|\right) \\
         \nonumber
         + \left\| \left(\mI - \left( \mA \mA^T+\mdelta\right)\mD(\vm) \right)^{-1}\mdelta \mA\vz_0  \right\|.
         \label{eq:prf-thm-inp-d-dim-noise-2}
    \end{eqnarray}
    Since $\hat{\lambda}_{\max}$ denotes the largest eigen value of $ \left( \mA \mA^T+\mdelta\right) \mD \left( \vm\right)$, we further simplify as:
    \begin{eqnarray}
    \nonumber
        \left\|\overleftarrow{\vx_0^r} - \rxz\right\| 
        \leq
        \\
        \nonumber
        \hat{\lambda}_{\max}^r 
        \left(\left\| \lxz - \mA\vz_0\right\|+ \left\| \left(\mI - \hat{\lambda}_{\max}\mI\right)^{-1}\mdelta \mA\vz_0
         \right\|\right)
         + \left\| \left(\mI - \hat{\lambda}_{\max}\mI \right)^{-1}\mdelta \mA\vz_0  \right\|
         =
         \\
         \nonumber
         \hat{\lambda}_{\max}^r 
         \left( \left\| \lxz - \mA\vz_0\right\|
         +  \frac{\left\|\mdelta \mA\vz_0  \right\|}{\left(1 - \hat{\lambda}_{\max}\right)}
         \right)
         +  \frac{\left\|\mdelta \mA\vz_0  \right\|}{\left(1 - \hat{\lambda}_{\max}\right)}.
         \label{eq:prf-thm-inp-d-dim-noise-3}
    \end{eqnarray}
Using the fact that $\lxz = \mA\mA^T\lxo$, $\rxz = \mA\vz_0$, and $\mA^T\mA = \mI$, we get 
\begin{eqnarray}
\nonumber
    \left\|\overleftarrow{\vx_0^r} - \rxz\right\|
    \leq 
    \hat{\lambda}_{\max}^r
    \left(
    \left\| \mA\mA^T\lxo - \mA(\mA^T\mA)\vz_0\right\|
    + 
    \frac{\left\|\mdelta \mA\vz_0  \right\|}{\left(1 - \hat{\lambda}_{\max}\right)}
    \right)
    +  \frac{\left\|\mdelta \mA\vz_0  \right\|}{\left(1 - \hat{\lambda}_{\max}\right)}
    \leq 
    \\
    \nonumber
    \hat{\lambda}_{\max}^r
    \left(
    \left\| \mA\mA^T \right\| \left\| \lxo - \mA\vz_0\right\|
    + 
    \frac{\left\|\mdelta \right\| \left\|\mA\vz_0  \right\|}{\left(1 - \hat{\lambda}_{\max}\right)}
    \right)
    +  \frac{\left\|\mdelta\right\| \left\| \mA\vz_0  \right\|}{\left(1 - \hat{\lambda}_{\max}\right)} 
    =
    \\
    \nonumber
    \hat{\lambda}_{\max}^r
    \left(
    \frac{\left\| \vtheta^* \right\|\left\| \lxo - \rxz\right\|}{\sqrt{1-\beta}} 
    + 
    \frac{\left\|\vtheta - \vtheta^* \right\| \left\|\rxz  \right\|}{\left(1 - \hat{\lambda}_{\max}\right)\sqrt{1-\beta}}
    \right)
    +  \frac{\left\|\vtheta - \vtheta^* \right\| \left\|\rxz  \right\|}{\left(1 - \hat{\lambda}_{\max}\right)\sqrt{1-\beta}},
\end{eqnarray}
which completes the statement of the theorem. \hfill $\square$

\subsection{Proof of \textbf{Corollary~\ref{cor:thm-inp-d-dim-noise}}}
\label{subsec:prf-cor-inp-d-dim}}
We begin with the statement of \textbf{Theorem~\ref{thm:inp-d-dim-noise}}, 
    \begin{eqnarray}
        \nonumber
        \left\|\overleftarrow{\vx_0^r} - \rxz\right\|
    \leq
    \hat{\lambda}_{\max}^r
    \left(
    \frac{\left\| \vtheta^* \right\|\left\| \lxo - \rxz\right\|}{\sqrt{1-\beta}} 
    + 
    \frac{\left\|\vtheta - \vtheta^* \right\| \left\|\rxz  \right\|}{\left(1 - \hat{\lambda}_{\max}\right)\sqrt{1-\beta}}
    \right)
    +  \frac{\left\|\vtheta - \vtheta^* \right\| \left\|\rxz  \right\|}{\left(1 - \hat{\lambda}_{\max}\right)\sqrt{1-\beta}}.
    \end{eqnarray}
    Since $\hat{\lambda}_{\max} < 1$, the first term vanishes in the limit. Using the fact that $q(\rxz)$ has a compact support, we get
    \begin{align*}
        \mathop{\sup}_{\substack{\rxz \sim q(\rxz)}} \left\|\overleftarrow{\vx_0^r} - \rxz\right\|
    & \leq
    \mathop{\sup}_{\substack{\rxz \sim q(\rxz)}}
    \frac{\left\|\vtheta - \vtheta^* \right\| \left\|\rxz  \right\|}{\left(1 - \hat{\lambda}_{\max}\right)\sqrt{1-\beta}}\\
     & \leq  \frac{\kappa\left\|\vtheta - \vtheta^* \right\| }{\left(1 - \hat{\lambda}_{\max}\right)\sqrt{1-\beta}} \\
     & =  \frac{\kappa \delta}{\left(1 - \hat{\lambda}_{\max}\right)}.
    \end{align*}
    Substituting $\delta=\gO\left(\epsilon \left(1-\hat{\lambda}_{\max} \right)/\kappa\right)$ for a $\delta$-approximate generator, we have $\mathop{\sup}_{\substack{\rxz \sim q(\rxz)}} \left\|\overleftarrow{\vx_0^r} - \rxz\right\|
    \leq \epsilon$, which finishes the proof. \hfill $\square$

\subsection{Proof of \textbf{Theorem~\ref{thm:gen-d-dim-t-states}}}
\label{subsec:prf-thm-gen-d-dim-t-state}
For diffusion over $T+1$ states, the training objective to minimize the difference in posterior means is given by (\ref{eq:vlbo_mu}):
\begin{eqnarray}
\nonumber
\min_{\vtheta} \mathop{\mathbb{E}}_{\substack{\rxz\sim q(\rxz)\\ \vepsilon\sim\gN\left(\0,\mI \right)\\ t\sim \gU\left( 1,\dots,T\right)}}\left [ \left \| \tilde{\mu}_t\left ( \overrightarrow{\vx_t}, \overrightarrow{\vx_0} \right ) - \mu_\theta\left ( \overleftarrow{\vx_t}, t \right ) \right \|^2
\right ] 
\coloneqq
\left [ \left \| \frac{\sqrt{\Bar{\alpha}_{t-1} }\beta_t}{1-\Bar{a}_t}\overrightarrow{\vx_0} + \frac{\sqrt{\alpha}_t\left ( 1-\Bar{\alpha}_{t-1} \right )}{1-\Bar{\alpha}_t} \overrightarrow{\vx_t }
 - \mu_\vtheta\left ( \overleftarrow{\vx_t}, t \right ) \right \|^2
\right ]
=
\\
\nonumber
\mathop{\mathbb{E}}_{\substack{\rxz\sim q(\rxz)\\ \vepsilon\sim\gN\left(\0,\mI \right)\\  t\sim \gU\left( 1,\dots,T\right)}} 
\left [ \left \| \frac{\sqrt{\Bar{\alpha}_{t-1} }\beta_t}{1-\Bar{a}_t}\overrightarrow{\vx_0} + \frac{\sqrt{\alpha}_t\left ( 1-\Bar{\alpha}_{t-1} \right )}{1-\Bar{\alpha}_t} \overrightarrow{\vx_t }
 - \vtheta\begin{bmatrix}
     \overleftarrow{\vx_t}\\ t
 \end{bmatrix}  \right \|^2
\right ].
\label{eq:thm-gen-T-states}
\end{eqnarray}

We know from our previous analysis that the closed-form solution of the above MMSE problem
is expressed as: 
\begin{eqnarray}
\nonumber
    \vtheta_i^* 
    =
    \left[    
    \mathop{\mathbb{E}}_{\substack{\rxz\sim q(\rxz)\\ \vepsilon\sim\gN\left(\0,\mI \right)\\  t\sim \gU\left( 1,\dots,T\right)}}
        \begin{bmatrix}
            \overleftarrow{\vx_t}\\ t
        \end{bmatrix}
        \begin{bmatrix}
            \overleftarrow{\vx_t}\\ t
        \end{bmatrix}^T  
    \right]^{-1}
    \times
    \mathop{\mathbb{E}}_{\substack{\rxz\sim q(\rxz)\\ \vepsilon\sim\gN\left(\0,\mI \right)\\  t\sim \gU\left( 1,\dots,T\right)}}
    \left[
        \left( \frac{\sqrt{\Bar{\alpha}_{t-1} }\beta_t}{1-\Bar{a}_t}\overrightarrow{\vx_0} + \frac{\sqrt{\alpha}_t\left ( 1-\Bar{\alpha}_{t-1} \right )}{1-\Bar{\alpha}_t} \overrightarrow{\vx_t } \right)_i 
        \begin{bmatrix}
            \overleftarrow{\vx_t}\\ t
        \end{bmatrix}
    \right].
\end{eqnarray}
By substituting the expression for samples generated at intermediate states, we get the following result:
    \begin{eqnarray}
    \nonumber
    \vtheta_i^* 
    =
    \left[
    \mathop{\mathbb{E}}_{\substack{\rxz\sim q(\rxz)\\ \vepsilon\sim\gN\left(\0,\mI \right)\\  t\sim \gU\left( 1,\dots,T\right)}}
        \begin{bmatrix}
            \rxz \sqrt{\Bar{\alpha}_t}  + \vepsilon\sqrt{1 - \Bar{\alpha}_t} \\ t
        \end{bmatrix}
        \begin{bmatrix}
            \rxz \sqrt{\Bar{\alpha}_t}  + \vepsilon\sqrt{1 - \Bar{\alpha}_t} \\ t
        \end{bmatrix}^T  
    \right]^{-1} \\
    \nonumber
    \times
    \mathop{\mathbb{E}}_{\substack{\rxz\sim q(\rxz)\\ \vepsilon\sim\gN\left(\0,\mI \right)\\  t\sim \gU\left( 1,\dots,T\right)}}
    \left[
        \left( \frac{\sqrt{\Bar{\alpha}_{t-1} }\beta_t}{1-\Bar{a}_t}\overrightarrow{\vx_0} + \frac{\sqrt{\alpha}_t\left ( 1-\Bar{\alpha}_{t-1} \right )}{1-\Bar{\alpha}_t} \left(\rxz \sqrt{\Bar{\alpha}_t}  + \vepsilon\sqrt{1 - \Bar{\alpha}_t} \right) \right)_i 
        \begin{bmatrix}
            \rxz \sqrt{\Bar{\alpha}_t}  + \vepsilon\sqrt{1 - \Bar{\alpha}_t}\\ t
        \end{bmatrix}
    \right].
    \end{eqnarray}
    Taking the outer product and expanding further, we arrive at
    \begin{eqnarray}
    \nonumber
    \vtheta_i^* 
    = 
    \left[
    \mathop{\mathbb{E}}_{\substack{\rxz\sim q(\rxz)\\ \vepsilon\sim\gN\left(\0,\mI \right)\\  t\sim \gU\left( 1,\dots,T\right)}}
        \begin{bmatrix}
            \left(\rxz \sqrt{\Bar{\alpha}_t}  + \vepsilon\sqrt{1 - \Bar{\alpha}_t}\right) \left(\rxz \sqrt{\Bar{\alpha}_t}  + \vepsilon\sqrt{1 - \Bar{\alpha}_t}\right)^T
            & \left(\rxz \sqrt{\Bar{\alpha}_t}  + \vepsilon\sqrt{1 - \Bar{\alpha}_t}\right) t 
            \\ 
            t \left(\rxz \sqrt{\Bar{\alpha}_t}  + \vepsilon\sqrt{1 - \Bar{\alpha}_t}\right)
            & 
            t^2
        \end{bmatrix}  
    \right]^{-1} \\
    \nonumber
    \times
    \mathop{\mathbb{E}}_{\substack{\rxz\sim q(\rxz)\\ \vepsilon\sim\gN\left(\0,\mI \right)\\  t\sim \gU\left( 1,\dots,T\right)}}
    \left[
        \left( \rxz\left(\frac{\sqrt{\Bar{\alpha}_{t-1} }\beta_t}{1-\Bar{a}_t}
        + \frac{\sqrt{\Bar{\alpha}_t} \sqrt{\alpha}_t\left ( 1-\Bar{\alpha}_{t-1} \right )}{1-\Bar{\alpha}_t} \right)  
        + 
        \vepsilon
            \left(
                \frac{\sqrt{\alpha}_t\left ( 1-\Bar{\alpha}_{t-1} \right )\sqrt{1 - \Bar{\alpha}_t}}{1-\Bar{\alpha}_t}
            \right) 
        \right)_i 
        \begin{bmatrix}
            \rxz \sqrt{\Bar{\alpha}_t}  + \vepsilon\sqrt{1 - \Bar{\alpha}_t}\\ t
        \end{bmatrix}
    \right],
    \end{eqnarray}
    which upon rearrangement gives the following expression:
    \begin{eqnarray}
    \nonumber
    \vtheta_i^* 
    =
    \\
    \nonumber
    \left[
    \mathop{\mathbb{E}}_{\substack{\rxz\sim q(\rxz)\\ \vepsilon\sim\gN\left(\0,\mI \right)\\  t\sim \gU\left( 1,\dots,T\right)}}
        \begin{bmatrix}
            \left(\rxz \rxz^T \Bar{\alpha}_t  
                + \rxz~\vepsilon^T~\sqrt{\Bar{\alpha}_t}\sqrt{1 - \Bar{\alpha}_t}
                + \vepsilon~\rxz^T~\sqrt{\Bar{\alpha}_t}\sqrt{1 - \Bar{\alpha}_t}
                + \vepsilon~\vepsilon^T \left(1-\Bar{\alpha}_t \right)
            \right) 
            & \left(\rxz \sqrt{\Bar{\alpha}_t}  + \vepsilon\sqrt{1 - \Bar{\alpha}_t}\right) t 
            \\ 
            t \left(\rxz \sqrt{\Bar{\alpha}_t}  + \vepsilon\sqrt{1 - \Bar{\alpha}_t}\right)
            & 
            t^2
        \end{bmatrix}  
    \right]^{-1} \\
    \nonumber
    \times
    \mathop{\mathbb{E}}_{\substack{\rxz\sim q(\rxz)\\ \vepsilon\sim\gN\left(\0,\mI \right)\\  t\sim \gU\left( 1,\dots,T\right)}}
    \left[
        \left( \rxz\left(\frac{\sqrt{\Bar{\alpha}_{t-1} }\beta_t}{1-\Bar{a}_t}
        + \frac{\sqrt{\Bar{\alpha}_t} \sqrt{\alpha}_t\left ( 1-\Bar{\alpha}_{t-1} \right )}{1-\Bar{\alpha}_t} \right)  
        + 
        \vepsilon
            \left(
                \frac{\sqrt{\alpha}_t\left ( 1-\Bar{\alpha}_{t-1} \right )\sqrt{1 - \Bar{\alpha}_t}}{1-\Bar{\alpha}_t}
            \right) 
        \right)_i 
        \begin{bmatrix}
            \rxz \sqrt{\Bar{\alpha}_t}  + \vepsilon\sqrt{1 - \Bar{\alpha}_t}\\ t
        \end{bmatrix}
    \right].
    \label{eq:thm-gen-T-states-1}
\end{eqnarray}
 Recall that $\rxz$ and $\vepsilon$ have zero means. Now, let us use the fact that $\rxz$, $\vepsilon$, and $t$ are all independent random variables. 
\begin{eqnarray}
    \nonumber
    \vtheta_i^* 
    =
    \left[
    \mathop{\mathbb{E}}_{\substack{\vz_0\sim \gN(\0,\mI)\\ \vepsilon\sim\gN\left(\0,\mI \right)\\  t\sim \gU\left( 1,\dots,T\right)}}
        \begin{bmatrix}
            \left(\mA\vz_0 \left(\mA\vz_0 \right)^T \Bar{\alpha}_t  
                + \vepsilon~\vepsilon^T \left(1-\Bar{\alpha}_t \right)
            \right) 
            & \0 
            \\ 
            \0
            & 
            t^2
        \end{bmatrix}  
    \right]^{-1} \\
    \nonumber
    \times
    \mathop{\mathbb{E}}_{\substack{\rxz\sim q(\rxz)\\ \vepsilon\sim\gN\left(\0,\mI \right)\\  t\sim \gU\left( 1,\dots,T\right)}}
    \left[
        \left( \rxz\left(\frac{\sqrt{\Bar{\alpha}_{t-1} }\beta_t}{1-\Bar{a}_t}
        + \frac{\sqrt{\Bar{\alpha}_t} \sqrt{\alpha}_t\left ( 1-\Bar{\alpha}_{t-1} \right )}{1-\Bar{\alpha}_t} \right)  
        + 
        \vepsilon
            \left(
                \frac{\sqrt{\alpha}_t\left ( 1-\Bar{\alpha}_{t-1} \right )\sqrt{1 - \Bar{\alpha}_t}}{1-\Bar{\alpha}_t}
            \right) 
        \right)_i 
        \begin{bmatrix}
            \rxz \sqrt{\Bar{\alpha}_t}  + \vepsilon\sqrt{1 - \Bar{\alpha}_t}\\ t
        \end{bmatrix}
    \right].
    \end{eqnarray}
    Since $\alpha_t = 1- \beta_t$ and $\sqrt{\Bar{\alpha}_t} = \sqrt{\Bar{\alpha}_{t-1}} \sqrt{\alpha_t}$, we further simplify the closed-form solution as follows:
    \begin{eqnarray}
    \nonumber
    \vtheta_i^*
    =
    \\
    \nonumber
    \left[
    \mathop{\mathbb{E}}_{\substack{\vz_0\sim \gN(\0,\mI)\\ \vepsilon\sim\gN\left(\0,\mI \right)\\  t\sim \gU\left( 1,\dots,T\right)}}
        \begin{bmatrix}
            \left(\mA\vz_0\vz_0^T\mA^T \Bar{\alpha}_t  
                + \vepsilon~\vepsilon^T \left(1-\Bar{\alpha}_t \right)
            \right) 
            & \0 
            \\ 
            \0
            & 
            t^2
        \end{bmatrix}  
    \right]^{-1} \\
    \nonumber
    \times
    \mathop{\mathbb{E}}_{\substack{\rxz\sim q(\rxz)\\ \vepsilon\sim\gN\left(\0,\mI \right)\\  t\sim \gU\left( 1,\dots,T\right)}}
    \left[
        \left( \rxz\left(
        \frac{\sqrt{\Bar{\alpha}_{t-1}}
        \left(\beta_t + \alpha_t \left ( 1-\Bar{\alpha}_{t-1} \right )
        \right)
        }
        {1-\Bar{a}_t}
        \right)  
        + 
        \vepsilon
            \left(
                \frac{\sqrt{\alpha}_t\left ( 1-\Bar{\alpha}_{t-1} \right )}{\sqrt{1 - \Bar{\alpha}_t}}
            \right) 
        \right)_i 
        \begin{bmatrix}
            \rxz \sqrt{\Bar{\alpha}_t}  + \vepsilon\sqrt{1 - \Bar{\alpha}_t}\\ t
        \end{bmatrix}
    \right]
    =
    \\
    \nonumber
    \left[
    \mathop{\mathbb{E}}_{\substack{\vz_0\sim \gN(\0,\mI)\\ \vepsilon\sim\gN\left(\0,\mI \right)\\  t\sim \gU\left( 1,\dots,T\right)}}
        \begin{bmatrix}
            \left(\mA\mA^T \Bar{\alpha}_t  
                + \vepsilon~\vepsilon^T \left(1-\Bar{\alpha}_t \right)
            \right) 
            & \0 
            \\ 
            \0
            & 
            t^2
        \end{bmatrix}  
    \right]^{-1} \\
    \nonumber
    \times
    \mathop{\mathbb{E}}_{\substack{\rxz\sim q(\rxz)\\ \vepsilon\sim\gN\left(\0,\mI \right)\\  t\sim \gU\left( 1,\dots,T\right)}}
    \left[
        \left( \rxz\left(
        \frac{\sqrt{\Bar{\alpha}_{t-1}}
        \left(\beta_t + \alpha_t - \Bar{\alpha}_{t} 
        \right)
        }
        {1-\Bar{a}_t}
        \right)  
        + 
        \vepsilon
            \left(
                \frac{\sqrt{\alpha}_t\left ( 1-\Bar{\alpha}_{t-1} \right )}{\sqrt{1 - \Bar{\alpha}_t}}
            \right) 
        \right)_i 
        \begin{bmatrix}
            \rxz \sqrt{\Bar{\alpha}_t}  + \vepsilon\sqrt{1 - \Bar{\alpha}_t}\\ t
        \end{bmatrix}
    \right]
    =
    \\
    \nonumber
    \mathop{\mathbb{E}}_{\substack{   t\sim \gU\left( 1,\dots,T\right)}}
    \begin{bmatrix}
            \left(\mA\mA^T \Bar{\alpha}_t  
                + \mI \left(1-\Bar{\alpha}_t \right)
            \right) 
            & \0 
            \\ 
            \0
            & 
            \frac{T^3- 1 }{3T}
        \end{bmatrix}^{-1} \\
        \nonumber
    \times
    \mathop{\mathbb{E}}_{\substack{\rxz\sim q(\rxz)\\ \vepsilon\sim\gN\left(\0,\mI \right)\\  t\sim \gU\left( 1,\dots,T\right)}}
    \left[
        \left(
        \rxz\left(
        \frac{\sqrt{\Bar{\alpha}_{t-1}}
        \left(\beta_t + 1 - \beta_t - \Bar{\alpha}_{t} 
        \right)
        }
        {1-\Bar{a}_t}
        \right)  
        + 
        \vepsilon
            \left(
                \frac{\sqrt{\alpha}_t\left ( 1-\Bar{\alpha}_{t-1} \right )}{\sqrt{1 - \Bar{\alpha}_t}}
            \right) 
        \right)_i 
        \begin{bmatrix}
            \rxz \sqrt{\Bar{\alpha}_t}  + \vepsilon\sqrt{1 - \Bar{\alpha}_t}\\ t
        \end{bmatrix}
    \right]
    =
    \\
    \nonumber
    \mathop{\mathbb{E}}_{\substack{   t\sim \gU\left( 1,\dots,T\right)}}
    \begin{bmatrix}
            \left[\mA\mA^T \Bar{\alpha}_t  
                + \mI \left(1-\Bar{\alpha}_t \right) 
            \right]^{-1}
            & \0 
            \\ 
            \0
            & 
            \frac{3T}{T^3- 1}
        \end{bmatrix}
    \\
    \nonumber
    \times
    \mathop{\mathbb{E}}_{\substack{\rxz\sim q(\rxz)\\ \vepsilon\sim\gN\left(\0,\mI \right)\\  t\sim \gU\left( 1,\dots,T\right)}}
    \left[
        \left(
        \rxz \sqrt{\Bar{\alpha}_{t-1}}
        + 
        \vepsilon
            \left(
                \frac{\sqrt{\alpha}_t\left ( 1-\Bar{\alpha}_{t-1} \right )}{\sqrt{1 - \Bar{\alpha}_t}}
            \right) 
        \right)_i 
        \begin{bmatrix}
            \rxz \sqrt{\Bar{\alpha}_t}  + \vepsilon\sqrt{1 - \Bar{\alpha}_t}\\ t
        \end{bmatrix}
    \right].
    \label{eq:thm-gen-T-states-2}
\end{eqnarray}

Using Woodburry matrix inverse~\cite{hager1989updating} in the above expression, we get 

\begin{eqnarray}
\nonumber
    \vtheta_i^*
    = \mathop{\mathbb{E}}_{\substack{   t\sim \gU\left( 1,\dots,T\right)}}
    \begin{bmatrix}
            \left(1-\Bar{\alpha}_t \right)^{-1}
            \left[
            \mI + \left(\mA\sqrt{\frac{\Bar{\alpha}_t}{\left(1-\Bar{\alpha}_t \right)}}\right)
             \left(\mA\sqrt{\frac{\Bar{\alpha}_t}{\left(1-\Bar{\alpha}_t \right)}}\right)^T   
            \right]^{-1}
            & \0 
            \\ 
            \0
            & 
            \frac{3T}{T^3- 1}
        \end{bmatrix}\\
        \nonumber
    \times
    \mathop{\mathbb{E}}_{\substack{\vz_0 \sim \gN(\0,\mI)\\ \vepsilon\sim\gN\left(\0,\mI \right)\\ t\sim \gU\left( 1,\dots,T\right)}}
    \left[
        \left(
        \mA\vz_0 \sqrt{\Bar{\alpha}_{t-1}}
        + 
        \vepsilon
            \left(
                \frac{\sqrt{\alpha}_t\left ( 1-\Bar{\alpha}_{t-1} \right )}{\sqrt{1 - \Bar{\alpha}_t}}
            \right) 
        \right)_i 
        \begin{bmatrix}
            \mA\vz_0 \sqrt{\Bar{\alpha}_t}  + \vepsilon\sqrt{1 - \Bar{\alpha}_t}\\ t
        \end{bmatrix}
    \right]
    =
    \\
    \nonumber
    \mathop{\mathbb{E}}_{\substack{   t\sim \gU\left( 1,\dots,T\right)}}
    \begin{bmatrix}
            \left(1-\Bar{\alpha}_t \right)^{-1}
            \left[
            \mI - \left(\mA\sqrt{\frac{\Bar{\alpha}_t}{\left(1-\Bar{\alpha}_t \right)}}\right)
                \left(\mI+\left(\mA\sqrt{\frac{\Bar{\alpha}_t}{\left(1-\Bar{\alpha}_t \right)}}\right)^T\left(\mA\sqrt{\frac{\Bar{\alpha}_t}{\left(1-\Bar{\alpha}_t \right)}}\right)
                \right)^{-1}
             \left(\mA\sqrt{\frac{\Bar{\alpha}_t}{\left(1-\Bar{\alpha}_t \right)}}\right)^T   
            \right]^{-1}
            & \0 
            \\ 
            \0
            & 
            \frac{3T}{T^3- 1}
        \end{bmatrix}\\
        \nonumber
    \times
    \mathop{\mathbb{E}}_{\substack{\vz_0 \sim \gN(\0,\mI)\\ \vepsilon\sim\gN\left(\0,\mI \right)\\ t\sim \gU\left( 1,\dots,T\right)}}
    \left[
        \left(
        \mA\vz_0 \sqrt{\Bar{\alpha}_{t-1}}
        + 
        \vepsilon
            \left(
                \frac{\sqrt{\alpha}_t\left ( 1-\Bar{\alpha}_{t-1} \right )}{\sqrt{1 - \Bar{\alpha}_t}}
            \right) 
        \right)_i 
        \begin{bmatrix}
            \mA\vz_0 \sqrt{\Bar{\alpha}_t}  + \vepsilon\sqrt{1 - \Bar{\alpha}_t}\\ t
        \end{bmatrix}
    \right].
    \end{eqnarray}
    Next, we simplify the second term further by extracting the $i^{th}$ component $\left(
        \mA\vz_0 \sqrt{\Bar{\alpha}_{t-1}}
        + 
        \vepsilon
            \left(
                \frac{\sqrt{\alpha}_t\left ( 1-\Bar{\alpha}_{t-1} \right )}{\sqrt{1 - \Bar{\alpha}_t}}
            \right) 
        \right)_i$ as:
    \begin{eqnarray}
    \nonumber
    \vtheta_i^*
    =\\
    \nonumber
    \mathop{\mathbb{E}}_{\substack{   t\sim \gU\left( 1,\dots,T\right)}}
    \begin{bmatrix}
            \left(1-\Bar{\alpha}_t \right)^{-1}
            \left[
            \mI - \left(\mA\sqrt{\frac{\Bar{\alpha}_t}{\left(1-\Bar{\alpha}_t \right)}}\right)
                \left(\mI+\left(\mA\sqrt{\frac{\Bar{\alpha}_t}{\left(1-\Bar{\alpha}_t \right)}}\right)^T\left(\mA\sqrt{\frac{\Bar{\alpha}_t}{\left(1-\Bar{\alpha}_t \right)}}\right)
                \right)^{-1}
             \left(\mA\sqrt{\frac{\Bar{\alpha}_t}{\left(1-\Bar{\alpha}_t \right)}}\right)^T   
            \right]^{-1}
            & \0 
            \\ 
            \0
            & 
            \frac{3T}{T^3- 1}
        \end{bmatrix}\\
        \nonumber
    \times
    \mathop{\mathbb{E}}_{\substack{\vz_0 \sim \gN(\0,\mI)\\ \vepsilon\sim\gN\left(\0,\mI \right)\\  t\sim \gU\left( 1,\dots,T\right)}}
    \left[
        \left(
        \va_i^T\vz_0 \sqrt{\Bar{\alpha}_{t-1}}
        + 
        \vepsilon_i
            \left(
                \frac{\sqrt{\alpha}_t\left ( 1-\Bar{\alpha}_{t-1} \right )}{\sqrt{1 - \Bar{\alpha}_t}}
            \right) 
        \right)
        \begin{bmatrix}
            \mA\vz_0 \sqrt{\Bar{\alpha}_t}  + \vepsilon\sqrt{1 - \Bar{\alpha}_t}\\ t
        \end{bmatrix}
    \right]
    =
    \\
    \nonumber
    \mathop{\mathbb{E}}_{\substack{t\sim \gU\left( 1,\dots,T\right)}}
    \begin{bmatrix}
            \left(1-\Bar{\alpha}_t \right)^{-1}
            \left[
            \mI - \left(\mA\sqrt{\frac{\Bar{\alpha}_t}{\left(1-\Bar{\alpha}_t \right)}}\right)
                \left(\mI+ \mI \left(\frac{\Bar{\alpha}_t}{1-\Bar{\alpha}_t }\right)
                \right)^{-1}
             \left(\mA\sqrt{\frac{\Bar{\alpha}_t}{\left(1-\Bar{\alpha}_t \right)}}\right)^T   
            \right]^{-1}
            & \0 
            \\ 
            \0
            & 
            \frac{3T}{T^3- 1}
        \end{bmatrix}\\
        \nonumber
    \times
    \mathop{\mathbb{E}}_{\substack{\vz_0 \sim \gN(\0,\mI)\\ \vepsilon\sim\gN\left(\0,\mI \right)\\t\sim \gU\left( 1,\dots,T\right)}}
        \begin{bmatrix}
            \left(
        \va_i^T\vz_0 \sqrt{\Bar{\alpha}_{t-1}}
        + 
        \vepsilon_i
            \left(
                \frac{\sqrt{\alpha}_t\left ( 1-\Bar{\alpha}_{t-1} \right )}{\sqrt{1 - \Bar{\alpha}_t}}
            \right) 
        \right)
        \left(
        \mA\vz_0 \sqrt{\Bar{\alpha}_t}  + \vepsilon\sqrt{1 - \Bar{\alpha}_t}
        \right)
        \\ 
        \left(
        \va_i^T\vz_0 \sqrt{\Bar{\alpha}_{t-1}}
        + 
        \vepsilon_i
            \left(
                \frac{\sqrt{\alpha}_t\left ( 1-\Bar{\alpha}_{t-1} \right )}{\sqrt{1 - \Bar{\alpha}_t}}
            \right) 
        \right)
        t
        \end{bmatrix}
        =
        \\
        \nonumber
    \mathop{\mathbb{E}}_{\substack{   t\sim \gU\left( 1,\dots,T\right)}}
    \begin{bmatrix}
            \left(1-\Bar{\alpha}_t \right)^{-1}
            \left[
            \mI - \left(\mA\sqrt{\frac{\Bar{\alpha}_t}{\left(1-\Bar{\alpha}_t \right)}}\right)
                \mI\left(1-\Bar{\alpha}_t 
                \right)
             \left(\mA\sqrt{\frac{\Bar{\alpha}_t}{\left(1-\Bar{\alpha}_t \right)}}\right)^T   
            \right]
            & \0 
            \\ 
            \0
            & 
            \frac{3T}{T^3- 1}
        \end{bmatrix}\\
        \nonumber
    \times
    \mathop{\mathbb{E}}_{\substack{\vz_0 \sim \gN(\0,\mI)\\ \vepsilon\sim\gN\left(\0,\mI \right)\\ t\sim \gU\left( 1,\dots,T\right)}}
        \begin{bmatrix}
        \left(
        \mA\vz_0 (\va_i^T\vz_0) \sqrt{\Bar{\alpha}_{t-1}} \sqrt{\Bar{\alpha}_t}  
        +
        \vepsilon_i
            \left(
                \frac{\sqrt{\alpha}_t\left ( 1-\Bar{\alpha}_{t-1} \right )}{\sqrt{1 - \Bar{\alpha}_t}}
            \right) 
            \vepsilon\sqrt{1 - \Bar{\alpha}_t}
        \right)
        \\ 
        0
        \end{bmatrix}
        =
        \\
        \nonumber
        \mathop{\mathbb{E}}_{\substack{  t\sim \gU\left( 1,\dots,T\right)}}
        \begin{bmatrix}
            \left(1-\Bar{\alpha}_t \right)^{-1}
            \left[
            \mI - \mA\mA^T\Bar{\alpha}_t  
            \right]
            & \0 
            \\ 
            \0
            & 
            \frac{3T}{T^3- 1}
        \end{bmatrix}\\
        \nonumber
    \times
    \mathop{\mathbb{E}}_{\substack{ \vepsilon\sim\gN\left(\0,\mI \right)\\  t\sim \gU\left( 1,\dots,T\right)}}
        \begin{bmatrix}
        \left(
        \mA\va_i\Bar{\alpha}_{t-1} \sqrt{\alpha_t}  
        +
        \vepsilon_i
            \left(
                \sqrt{\alpha}_t\left ( 1-\Bar{\alpha}_{t-1} \right )
            \right) 
            \vepsilon
        \right)
        \\  
        0
        \end{bmatrix}
        =
        \\
        \nonumber
        \mathop{\mathbb{E}}_{\substack{  t\sim \gU\left( 1,\dots,T\right)}}
        \begin{bmatrix}
            \frac{1}{\left(1-\Bar{\alpha}_t \right)}
            \left[
            \mI - \mA\mA^T\Bar{\alpha}_t  
            \right]
            & \0 
            \\ 
            \0
            & 
            \frac{3T}{T^3- 1}
        \end{bmatrix}\\
        \nonumber
    \times
    \mathop{\mathbb{E}}_{\substack{ t\sim \gU\left( 1,\dots,T\right)}}
        \begin{bmatrix}
        \mA\va_i\Bar{\alpha}_{t-1} \sqrt{\alpha_t}  
        +
        \ve_i
       \left(
                \sqrt{\alpha}_t\left ( 1-\Bar{\alpha}_{t-1} \right )
            \right) 
            \\  
        0
        \end{bmatrix},
\end{eqnarray}
where $\ve_i \in \{0,1 \}^d$ is a one-hot encoded vector with only $i^{th}$ coordinate taking the value $1$. Bringing the expectation inside and using matrix products, we obtain
\begin{eqnarray}
\nonumber
    \vtheta_i^*
    =
            \begin{bmatrix}
                \mathop{\mathbb{E}}_{\substack{t}}
                \left(
                    \frac{1}{\left(1-\Bar{\alpha}_t \right)}
                \right) 
                \mI
                -
                \mathop{\mathbb{E}}_{\substack{t}}
                \left(
                    \frac{\Bar{\alpha}_t}{\left(1-\Bar{\alpha}_t \right)}
                \right)
                \mA\mA^T
            & \0 
            \\ 
            \0
            & 
            \frac{3T}{T^3- 1}
        \end{bmatrix}
        \begin{bmatrix}
        \mathop{\mathbb{E}}_{\substack{t}}
         \left(
         \Bar{\alpha}_{t-1} \sqrt{\alpha_t} 
         \right) 
         \mA\va_i
        +
        \mathop{\mathbb{E}}_{\substack{t}}
         \left(
            \sqrt{\alpha}_t\left ( 1-\Bar{\alpha}_{t-1} \right )
         \right)
        \ve_i
            \\  
        0
        \end{bmatrix}
        =
        \\
        \nonumber
        \begin{bmatrix}
         \left(
        \mathop{\mathbb{E}}_{\substack{t}}
                \left(
                    \frac{1}{\left(1-\Bar{\alpha}_t \right)}
                \right) 
                \mI
                -
                \mathop{\mathbb{E}}_{\substack{t}}
                \left(
                    \frac{\Bar{\alpha}_t}{\left(1-\Bar{\alpha}_t \right)}
                \right)
                \mA\mA^T
        \right)
        \left(
        \mathop{\mathbb{E}}_{\substack{t}}
         \left(
         \Bar{\alpha}_{t-1} \sqrt{\alpha_t} 
         \right) 
         \mA\va_i
        +
        \mathop{\mathbb{E}}_{\substack{t}}
         \left(
            \sqrt{\alpha}_t\left ( 1-\Bar{\alpha}_{t-1} \right )
         \right)
        \ve_i
        \right)
            \\  
        0
        \end{bmatrix}
        =
        \\
        \nonumber
        \begin{bmatrix}
        \mathop{\mathbb{E}}_{\substack{t}}
                \left(
                    \frac{1}{\left(1-\Bar{\alpha}_t \right)}
                \right) 
        \mathop{\mathbb{E}}_{\substack{t}}
         \left(
         \Bar{\alpha}_{t-1} \sqrt{\alpha_t} 
         \right)
         \mA\va_i
         -
            \mathop{\mathbb{E}}_{\substack{t}}
                \left(
                    \frac{\Bar{\alpha}_t}{\left(1-\Bar{\alpha}_t \right)}
                \right)
                \mathop{\mathbb{E}}_{\substack{t}}
             \left(
             \Bar{\alpha}_{t-1} \sqrt{\alpha_t} 
             \right)
                  \mA\mA^T\mA\va_i
        \\
        +
        \mathop{\mathbb{E}}_{\substack{t}}
                \left(
                    \frac{1}{\left(1-\Bar{\alpha}_t \right)}
                \right) 
        \mathop{\mathbb{E}}_{\substack{t}}
         \left(
            \sqrt{\alpha}_t\left ( 1-\Bar{\alpha}_{t-1} \right )
         \right)
        \ve_i
        -
        \mathop{\mathbb{E}}_{\substack{t}}
                \left(
                    \frac{\Bar{\alpha}_t}{\left(1-\Bar{\alpha}_t \right)}
                \right)
        \mathop{\mathbb{E}}_{\substack{t}}
         \left(
            \sqrt{\alpha}_t\left ( 1-\Bar{\alpha}_{t-1} \right )
         \right)
         \left(\mA\mA^T \right)
        \ve_i
            \\  
        0
        \end{bmatrix}
        =
        \\
        \nonumber
        \begin{bmatrix}
         \left(
        \mathop{\mathbb{E}}_{\substack{t}}
                \left(
                    \frac{1}{\left(1-\Bar{\alpha}_t \right)}
                \right) 
        \mathop{\mathbb{E}}_{\substack{t}}
         \left(
         \Bar{\alpha}_{t-1} \sqrt{\alpha_t} 
         \right)
         -
                \mathop{\mathbb{E}}_{\substack{t}}
                \left(
                    \frac{\Bar{\alpha}_t}{\left(1-\Bar{\alpha}_t \right)}
                \right)
                \mathop{\mathbb{E}}_{\substack{t}}
             \left(
             \Bar{\alpha}_{t-1} \sqrt{\alpha_t} 
             \right)
         
         \right)
         \mA\va_i
        \\
        +
        \mathop{\mathbb{E}}_{\substack{t}}
                \left(
                    \frac{1}{\left(1-\Bar{\alpha}_t \right)}
                \right) 
        \mathop{\mathbb{E}}_{\substack{t}}
         \left(
            \sqrt{\alpha}_t\left ( 1-\Bar{\alpha}_{t-1} \right )
         \right)
        \ve_i
        -
        \mathop{\mathbb{E}}_{\substack{t}}
                \left(
                    \frac{\Bar{\alpha}_t}{\left(1-\Bar{\alpha}_t \right)}
                \right)
        \mathop{\mathbb{E}}_{\substack{t}}
         \left(
            \sqrt{\alpha}_t\left ( 1-\Bar{\alpha}_{t-1} \right )
         \right)
         \left(\mA\mA^T \right)
        \ve_i
            \\  
        0
        \end{bmatrix}.
\end{eqnarray}
Stacking all the rows together, we obtain the closed-form solution for generative modeling:
\begin{eqnarray}
\nonumber
    \vtheta^* 
    = \begin{bmatrix} 
    \left(
        \mathop{\mathbb{E}}_{\substack{t}}
                \left(
                    \frac{1}{\left(1-\Bar{\alpha}_t \right)}
                \right) 
        \mathop{\mathbb{E}}_{\substack{t}}
         \left(
         \Bar{\alpha}_{t-1} \sqrt{\alpha_t} 
         \right)
         -
                \mathop{\mathbb{E}}_{\substack{t}}
                \left(
                    \frac{\Bar{\alpha}_t}{\left(1-\Bar{\alpha}_t \right)}
                \right)
                \mathop{\mathbb{E}}_{\substack{t}}
             \left(
             \Bar{\alpha}_{t-1} \sqrt{\alpha_t} 
             \right)
         \right)
         \mA\mA^T
         \\
         \nonumber
         +
        \mathop{\mathbb{E}}_{\substack{t}}
                \left(
                    \frac{1}{\left(1-\Bar{\alpha}_t \right)}
                \right) 
        \mathop{\mathbb{E}}_{\substack{t}}
         \left(
            \sqrt{\alpha}_t\left ( 1-\Bar{\alpha}_{t-1} \right )
         \right)
        \mI
        \\
        \nonumber
        -
        \mathop{\mathbb{E}}_{\substack{t}}
                \left(
                    \frac{\Bar{\alpha}_t}{\left(1-\Bar{\alpha}_t \right)}
                \right)
        \mathop{\mathbb{E}}_{\substack{t}}
         \left(
            \sqrt{\alpha}_t\left ( 1-\Bar{\alpha}_{t-1} \right )
         \right)
         \mA\mA^T
          & \0
         \end{bmatrix}
   =      
         \\
         \nonumber 
    \begin{bmatrix} 
       \left(
        \mathop{\mathbb{E}}_{\substack{t}}
                \left(
                    \frac{1}{\left(1-\Bar{\alpha}_t \right)}
                \right) 
        \mathop{\mathbb{E}}_{\substack{t}}
         \left(
         \Bar{\alpha}_{t-1} \sqrt{\alpha_t} 
         \right)
         -
                \mathop{\mathbb{E}}_{\substack{t}}
                \left(
                    \frac{\Bar{\alpha}_t}{\left(1-\Bar{\alpha}_t \right)}
                \right)
                \mathop{\mathbb{E}}_{\substack{t}}
             \left(
             \Bar{\alpha}_{t-1} \sqrt{\alpha_t} 
             \right)
         \right)
         \mA\mA^T
         \\ 
         \nonumber
         +
        \mathop{\mathbb{E}}_{\substack{t}}
                \left(
                    \frac{1}{\left(1-\Bar{\alpha}_t \right)}
                \right) 
        \mathop{\mathbb{E}}_{\substack{t}}
         \left(
            \sqrt{\alpha}_t\left ( 1-\Bar{\alpha}_{t-1} \right )
         \right)
        \mI
        \\
        \nonumber
        -
        \mathop{\mathbb{E}}_{\substack{t}}
                \left(
                    \frac{\Bar{\alpha}_t}{\left(1-\Bar{\alpha}_t \right)}
                \right)
        \mathop{\mathbb{E}}_{\substack{t}}
         \left(
            \sqrt{\alpha}_t
         \right)
                  \mA\mA^T
         +
         \mathop{\mathbb{E}}_{\substack{t}}
                \left(
                    \frac{\Bar{\alpha}_t}{\left(1-\Bar{\alpha}_t \right)}
                \right)
         \mathop{\mathbb{E}}_{\substack{t}}
         \left(
            \sqrt{\alpha}_t\Bar{\alpha}_{t-1}
         \right)
         \mA\mA^T
          & \0
         \end{bmatrix}
    =
         \\
         \nonumber
    \begin{bmatrix} 
    \left(
        \mathop{\mathbb{E}}_{\substack{t}}
                \left(
                    \frac{1}{\left(1-\Bar{\alpha}_t \right)}
                \right) 
        \mathop{\mathbb{E}}_{\substack{t}}
         \left(
         \Bar{\alpha}_{t-1} \sqrt{\alpha_t} 
         \right)
         -
         \mathop{\mathbb{E}}_{\substack{t}}
                \left(
                    \frac{\Bar{\alpha}_t}{\left(1-\Bar{\alpha}_t \right)}
                \right)
        \mathop{\mathbb{E}}_{\substack{t}}
         \left(
            \sqrt{\alpha}_t
         \right)
         \right)
         \mA\mA^T
         \\ 
         \nonumber
         +
        \mathop{\mathbb{E}}_{\substack{t}}
                \left(
                    \frac{1}{\left(1-\Bar{\alpha}_t \right)}
                \right) 
        \mathop{\mathbb{E}}_{\substack{t}}
         \left(
            \sqrt{\alpha}_t\left ( 1-\Bar{\alpha}_{t-1} \right )
         \right)
        \mI
         & \0
         \end{bmatrix}
    =
    \\
    \nonumber
     \begin{bmatrix} 
     \nu \mA \mA^T
     + 
     \gamma \mI
      & \0
         \end{bmatrix}.
\end{eqnarray}
This completes the proof of the theorem.  \hfill $\square$

\subsection{Proof of \textbf{Corollary~\ref{cor:gen-d-dim-t-states}}}
\label{subsec:prf-cor-gen-d-dim-t-state}
    Using the arguments from Appendix~\ref{subsec:prf-thm-gen-d-dim-t-state}, we express the closed-form solution of the MMSE problem (\ref{eq:thm-gen-T-states}) as:
    \begin{eqnarray}
        \nonumber
        \vtheta_i^* 
    =
    \left[    
    \mathop{\mathbb{E}}_{\substack{\rxz\sim q(\rxz)\\ \vepsilon\sim\gN\left(\0,\mI \right)\\  t\sim \gU\left( 1,\dots,T\right)}}
        \begin{bmatrix}
            \overleftarrow{\vx_t}\\ t
        \end{bmatrix}
        \begin{bmatrix}
            \overleftarrow{\vx_t}\\ t
        \end{bmatrix}^T  
    \right]^{-1}
    \times
    \mathop{\mathbb{E}}_{\substack{\rxz\sim q(\rxz)\\ \vepsilon\sim\gN\left(\0,\mI \right)\\  t\sim \gU\left( 1,\dots,T\right)}}
    \left[
        \left( \frac{\sqrt{\Bar{\alpha}_{t-1} }\beta_t}{1-\Bar{a}_t}\overrightarrow{\vx_0} + \frac{\sqrt{\alpha}_t\left ( 1-\Bar{\alpha}_{t-1} \right )}{1-\Bar{\alpha}_t} \overrightarrow{\vx_t } \right)_i 
        \begin{bmatrix}
            \overleftarrow{\vx_t}\\ t
        \end{bmatrix}
    \right]
    =
        \\
    \nonumber
        \mathop{\mathbb{E}}_{\substack{   t\sim \gU\left( 1,\dots,T\right)}}
    \begin{bmatrix}
            \left[\mA\mA^T \Bar{\alpha}_t  
                + \mI \left(1-\Bar{\alpha}_t \right) 
            \right]^{-1}
            & \0 
            \\ 
            \0
            & 
            \frac{3T}{T^3- 1}
        \end{bmatrix}\\
        \nonumber
    \times
    \mathop{\mathbb{E}}_{\substack{\rxz\sim q(\rxz)\\ \vepsilon\sim\gN\left(\0,\mI \right)\\  t\sim \gU\left( 1,\dots,T\right)}}
    \left[
        \left(
        \rxz \sqrt{\Bar{\alpha}_{t-1}}
        + 
        \overrightarrow{\vepsilon}
            \left(
                \frac{\sqrt{\alpha}_t\left ( 1-\Bar{\alpha}_{t-1} \right )}{\sqrt{1 - \Bar{\alpha}_t}}
            \right) 
        \right)_i 
        \begin{bmatrix}
            \rxz \sqrt{\Bar{\alpha}_t}  + \overleftarrow{\vepsilon}\sqrt{1 - \Bar{\alpha}_t}\\ t
        \end{bmatrix}
    \right]
    =
        \\
    \nonumber
    \mathop{\mathbb{E}}_{\substack{   t\sim \gU\left( 1,\dots,T\right)}}
    \begin{bmatrix}
            \left(1-\Bar{\alpha}_t \right)^{-1}
            \left[
            \mI - \left(\mA\sqrt{\frac{\Bar{\alpha}_t}{\left(1-\Bar{\alpha}_t \right)}}\right)
                \mI\left(1-\Bar{\alpha}_t 
                \right)
             \left(\mA\sqrt{\frac{\Bar{\alpha}_t}{\left(1-\Bar{\alpha}_t \right)}}\right)^T   
            \right]
            & \0 
            \\ 
            \0
            & 
            \frac{3T}{T^3- 1}
        \end{bmatrix}\\
        \nonumber
    \times
    \mathop{\mathbb{E}}_{\substack{\vz_0 \sim \gN(\0,\mI)\\ \vepsilon\sim\gN\left(\0,\mI \right)\\ t\sim \gU\left( 1,\dots,T\right)}}
        \begin{bmatrix}
        \left(
        \mA\vz_0 (\va_i^T\vz_0) \sqrt{\Bar{\alpha}_{t-1}} \sqrt{\Bar{\alpha}_t}  
        +
        \overrightarrow{\vepsilon_i}
            \left(
                \frac{\sqrt{\alpha}_t\left ( 1-\Bar{\alpha}_{t-1} \right )}{\sqrt{1 - \Bar{\alpha}_t}}
            \right) 
            \overleftarrow{\vepsilon}\sqrt{1 - \Bar{\alpha}_t}
        \right)
        \\ 
        0
        \end{bmatrix}
        =
        \\
        \nonumber
        \mathop{\mathbb{E}}_{\substack{  t\sim \gU\left( 1,\dots,T\right)}}
        \begin{bmatrix}
            \left(1-\Bar{\alpha}_t \right)^{-1}
            \left[
            \mI - \mA\mA^T\Bar{\alpha}_t  
            \right]
            & \0 
            \\ 
            \0
            & 
            \frac{3T}{T^3- 1}
        \end{bmatrix}
    \mathop{\mathbb{E}}_{\substack{  t\sim \gU\left( 1,\dots,T\right)}}
        \begin{bmatrix}
         \mA\va_i\Bar{\alpha}_{t-1} \sqrt{\alpha_t}  
        \\  
        0
        \end{bmatrix}
        =
        \\
        \nonumber
        \begin{bmatrix}
                \mathop{\mathbb{E}}_{\substack{t}}
                \left(
                    \frac{1}{\left(1-\Bar{\alpha}_t \right)}
                \right) 
                \mI
                -
                \mathop{\mathbb{E}}_{\substack{t}}
                \left(
                    \frac{\Bar{\alpha}_t}{\left(1-\Bar{\alpha}_t \right)}
                \right)
                \mA\mA^T
            & \0 
            \\ 
            \0
            & 
            \frac{3T}{T^3- 1}
        \end{bmatrix}
        \begin{bmatrix}
        \mathop{\mathbb{E}}_{\substack{t}}
         \left(
         \Bar{\alpha}_{t-1} \sqrt{\alpha_t} 
         \right) 
         \mA\va_i
            \\  
        0
        \end{bmatrix}
    =
            \\
        \nonumber
        \begin{bmatrix}
         \left(
        \mathop{\mathbb{E}}_{\substack{t}}
                \left(
                    \frac{1}{\left(1-\Bar{\alpha}_t \right)}
                \right) 
        \mathop{\mathbb{E}}_{\substack{t}}
         \left(
         \Bar{\alpha}_{t-1} \sqrt{\alpha_t} 
         \right)
         -
                \mathop{\mathbb{E}}_{\substack{t}}
                \left(
                    \frac{\Bar{\alpha}_t}{\left(1-\Bar{\alpha}_t \right)}
                \right)
                \mathop{\mathbb{E}}_{\substack{t}}
             \left(
             \Bar{\alpha}_{t-1} \sqrt{\alpha_t} 
             \right)
         
         \right)
         \mA\va_i
            \\  
        0
        \end{bmatrix}
    \end{eqnarray}
Now, stacking all the rows together as in Appendix~\ref{subsec:prf-thm-gen-d-dim-t-state}, we obtain the closed-form solution for DDPM generator as $ \vtheta^* 
    =
    \begin{bmatrix} 
     \bnu \mA \mA^T
      & \0
         \end{bmatrix}
$, which finishes the proof of the corollary.  \hfill $\square$

\subsection{Proof of \textbf{Theorem~\ref{thm:gen-d-dim-t-states-sample}}}
\label{subsec:prf-thm-gen-d-dim-t-states-sample}
To prove DDPM learns the distribution $q\left(\rxz \right)$, it is sufficient to show that $\lxz \sim p\left(\lxz \right)$ has the structure $\mA\overleftarrow{\vz_0}$, where $\overleftarrow{\vz_0} \sim \gN\left(\0,\mI\right)$. Using the closed-form solution from \textbf{Corollary~\ref{cor:gen-d-dim-t-states}}, we run one step of the reverse process starting from $\overleftarrow{\vx_T} \sim \gN\left(\0,\mI \right)$. Then, we rectify the misalignment in drift and dispersion of the reverse process. After recursively applying the reverse step for all the states, we obtain
\begin{eqnarray}
    \nonumber
    \lxz 
    =
    \left( \prod_{t=1}^{T} \left(\omega_t \bnu \aat\right)\right) \lxt 
    + \left( \prod_{t=1}^{T-1} \left(\omega_t \bnu \aat\right)\right) \xi_T \sqrt{\beta_T} \overleftarrow{\epsilon_{T}}
    + \left( \prod_{t=1}^{T-2} \left(\omega_t \bnu \aat\right)\right) \xi_T \sqrt{\beta_{T-1}} \overleftarrow{\epsilon_{T-1}}
    \\
    \nonumber
    + \cdots 
    + \left(\omega_1\bnu \aat\right) \left(\omega_2\bnu \aat\right) \xi_3 \sqrt{\beta_3} \overleftarrow{\vepsilon_3}
    + \left(\omega_1 \bnu \aat \right) \xi_2\sqrt{\beta_2} \overleftarrow{\vepsilon_2}. 
\end{eqnarray}
Since $\{\overleftarrow{\vepsilon_t}\}_{t > 1}$ are IID Gaussian random variables, we rewrite the above epxression as:
\begin{eqnarray}
    \nonumber
    \lxz 
    = \left( \prod_{t=1}^{T} \left(\omega_t \bnu \aat\right)\right) \lxt 
    + 
    \Bar{\Sigma}^{1/2} \leps,
\end{eqnarray}
where $\leps \sim \gauss$ and $\Bar{\Sigma}$ captures the sum of variances of the terms containing $\{\overleftarrow{\vepsilon_t}\}_{t>1}$. Recall that $\mA^T\mA = \mI$, the coefficients of drift $\omega_t = \omega = \frac{1}{\bnu\sqrt{2}^{1/T}}$ and dispersion $\xi_t = \sqrt{\frac{2^{(t-1)/T}}{2\beta_t(T-1)}}$. Now, let us compute the variances of each term: 
\begin{eqnarray}
    \nonumber
    \E
    \left[
        \left(
            \left(\omega_1 \bnu \aat \right) \xi_2\sqrt{\beta_2} \overleftarrow{\vepsilon_2} 
        \right)
        \left(
            \left(\omega_1 \bnu \aat \right) \xi_2\sqrt{\beta_2} \overleftarrow{\vepsilon_2} 
        \right)^T
    \right]
     = 
     \omega_1^2\bnu^2 \xi_2^2 \beta_2  \aat\E\left[\overleftarrow{\vepsilon_2}\overleftarrow{\vepsilon_2}^T \right]\aat 
     = \omega_1^2\bnu^2 \xi_2^2 \beta_2  \aat\aat
     \\
     \nonumber
     = \omega_1^2\bnu^2 \xi_2^2 \beta_2  \mA\mA^T 
     = \left(\frac{1}{\bnu\sqrt{2}^{1/T}}\right)^2 \bnu^2 \left( \frac{2^{1/T}}{2\beta_2(T-1)}\right)\beta_2 \aat
     = \frac{\aat}{2(T-1)}; 
     \\
     \nonumber
     \E
     \left[
        \left(\left(\omega_1\bnu \aat\right) \left(\omega_2\bnu \aat\right) \xi_3 \sqrt{\beta_3} \overleftarrow{\vepsilon_3}\right)
        \left(\left(\omega_1\bnu \aat\right) \left(\omega_2\bnu \aat\right) \xi_3 \sqrt{\beta_3} \overleftarrow{\vepsilon_3}\right)^T 
     \right]
     =
     \frac{\aat}{2(T-1)}; \cdots;\\
     \nonumber
     \E
     \left[
        \left(\left( \prod_{t=1}^{T-1} \left(\omega_t \bnu \aat\right)\right) \xi_T \sqrt{\beta_T} \overleftarrow{\epsilon_{T}}
\right)\left(\left( \prod_{t=1}^{T-1} \left(\omega_t \bnu \aat\right)\right) \xi_T \sqrt{\beta_T} \overleftarrow{\epsilon_{T}}
\right)^T
     \right]
     =
     \frac{\aat}{2(T-1)}.
\end{eqnarray}
Note that $\left(\aat\right)^2 \coloneqq \aat\aat = \aat$. Since there are $(T-1)$ terms, the total covariance $\Bar{\Sigma} = \frac{\aat}{2}$, which is equivalent to the following sample generated at the end of the reverse process:
\begin{eqnarray}
    \nonumber
    \lxz 
    = \left( \prod_{t=1}^{T} \left(\omega_t \bnu \aat\right)\right) \lxt 
    + 
    \frac{\aat}{\sqrt{2}} \leps
    =
    \aat \left( \prod_{t=1}^{T} \left(\frac{1}{\bnu \sqrt{2}^{1/T}} \right)\bnu \right) \lxt 
    + 
    \frac{\aat}{\sqrt{2}} \leps 
    \\
    \nonumber
    = 
    \frac{\aat}{\sqrt{2}} \lxt
    +
    \frac{\aat}{\sqrt{2}} \leps
    = \mA
        \left( 
            \frac{\mA^T\lxt}{\sqrt{2}} 
            +
            \frac{\mA^T\leps}{\sqrt{2}} 
        \right)
    =\mA\overleftarrow{\vz_0}.
\end{eqnarray}
Now, let us verify if $\lxz$ lies on the data manifold. To see this, it is sufficient to show $\lzz$ has zero mean and unit covariance. Therefore, we compute 
\begin{eqnarray}
    \nonumber
    \E
    \left[
        \lzz 
    \right]
    = 
    \E
    \left[
         \frac{\mA^T\lxt}{\sqrt{2}} 
            +
            \frac{\mA^T\leps}{\sqrt{2}}
    \right]
    =
    \frac{\mA^T\E\left[\lxt\right]}{\sqrt{2}} 
            +
    \frac{\mA^T\E\left[\leps\right]}{\sqrt{2}}
    = \0
\end{eqnarray}
and 
\begin{eqnarray*}
    \E
    \left[
        \lzz\lzz^T
    \right]
    =
    \E
    \left[
         \left(
         \frac{\mA^T\lxt}{\sqrt{2}} 
            +
            \frac{\mA^T\leps}{\sqrt{2}}
         \right)
         \left(
         \frac{\mA^T\lxt}{\sqrt{2}} 
            +
            \frac{\mA^T\leps}{\sqrt{2}}
         \right)^T
    \right]
    \stackrel{(i)}{=}
         \frac{\mA^T
         \E
        \left[
            \lxt\lxt^T
        \right]\mA
        }{2} 
            +
            \frac{\mA^T\E
    \left[\leps\leps^T
    \right]\mA}{2} 
    \stackrel{(ii)}{=} \mI,
\end{eqnarray*}
where (i) uses the fact that $\lxt$ and $\leps$ are independent zero-mean random variables, and (ii) relies on unit covariance of $\lxt$, $\leps$, and the property that $\mA^T\mA = \mI$. Thus, we finish the proof. \hfill $\square$

\subsection{Proof of \textbf{Theorem~\ref{thm:inp-d-dim-t-states}}}
\label{subsec:thm-inp-d-dim-t-states}
Besides the inpainting step at every intermediate state, the proof follows from the arguments in Appendix~\ref{subsec:prf-thm-gen-d-dim-t-states-sample}. Starting from the Gaussian prior $\overleftarrow{\vx_T} \sim \gauss$, We first compute
\begin{eqnarray*}
    \overleftarrow{\vx_{T-1}} = \omega_T \bnu \aat \lxt + \xi_T \sqrt{\beta_T} \overleftarrow{\vepsilon_T}.
\end{eqnarray*}
In the next step, we run the forward SDE to compute the known portions of the given image with the noise level at $(T-1)$:
\begin{eqnarray*}
    \overrightarrow{\vx_{T-1}} = \sqrt{\Bar{\alpha}_{T-1}} \rxz + \sqrt{1 - \Bar{\beta}_{T-1}} \overrightarrow{\vepsilon_T}.
\end{eqnarray*}
Then, we update $\overleftarrow{\vx_{T-1}}$ by combining known portions with the generated unknown portions of the image:
\begin{eqnarray*}
    \overleftarrow{\vx_{T-1}} = \mD(\vm)~\overleftarrow{\vx_{T-1}} + \mD(1-\vm)~\overrightarrow{\vx_{T-1}}
\end{eqnarray*}
This completes inpainting at the state $(T-1)$ of the reverse process. Recursively applying this over all the states, we obtain
\begin{eqnarray*}
    \lxz 
    =
    \left(
        \prod_{t = T-1}^{1} \omega_{T-t} \bnu \aat \dm 
    \right)
    \left(\omega_T \bnu \aat \right) \lxt
    +
    \\
    \left(
        \prod_{t = T-1}^{2} \omega_{T-t} \bnu \aat \dm 
    \right)
    \left( 
        \omega_{T-1}\bnu \aat\dom \sqrt{\Bar{\alpha}_{T-1}} \rxz
    \right)
    +
    \\
    \left(
        \prod_{t = T-1}^{3} \omega_{T-t} \bnu \aat \dm 
    \right)
    \left( 
        \omega_{T-2}\bnu \aat\dom \sqrt{\Bar{\alpha}_{T-2}} \rxz
    \right)+\cdots
    +
    \\
    \left( 
        \omega_{1}\bnu \aat\dm \sqrt{\Bar{\alpha}_{1}} \rxz
    \right)
    \left( 
        \omega_{2}\bnu \aat\dom \sqrt{\Bar{\alpha}_{2}} \rxz
    \right)
    + \left( 
        \omega_{1}\bnu \aat\dom \sqrt{\Bar{\alpha}_{1}} \rxz
    \right)
    +
    \\
    \left(
        \prod_{t = T-1}^{2} \omega_{T-t} \bnu \aat \dm 
    \right)
    \left( 
        \omega_{T-1}\bnu \aat\dom \sqrt{1 - \Bar{\alpha}_{T-1}} \overrightarrow{\vepsilon_{T-1}}
    \right)
    +
    \\
    \left(
        \prod_{t = T-1}^{3} \omega_{T-t} \bnu \aat \dm 
    \right)
    \left( 
        \omega_{T-2}\bnu \aat\dom \sqrt{1 - \Bar{\alpha}_{T-2}} \overrightarrow{\vepsilon_{T-2}}
    \right)+\cdots
    +
    \\
    \left( 
        \omega_{1}\bnu \aat\dm \sqrt{1 - \Bar{\alpha}_{1}} \overrightarrow{\vepsilon_{1}}
    \right)
    \left( 
        \omega_{2}\bnu \aat\dom \sqrt{1 - \Bar{\alpha}_{2}} \overrightarrow{\vepsilon_{2}}
    \right)
    +
    \\
     \left( 
        \omega_{1}\bnu \aat\dom \sqrt{1 - \Bar{\alpha}_{1}} \overrightarrow{\vepsilon_{1}}
    \right)
    +
    \\
    \left(
        \prod_{t = T-1}^{1} \omega_{T-t} \bnu \aat \dm 
    \right)
    \xi_T \sqrt{\beta_T} \overleftarrow{\vepsilon_T}
    +
    \\
    \left(
        \prod_{t = T-1}^{2} \omega_{T-t} \bnu \aat \dm 
    \right)
    \xi_{T-1} \sqrt{\beta_{T-1}} \overleftarrow{\vepsilon_{T-1}}
    +
    \\
    \left(
        \omega_{1} \bnu \aat \dm 
    \right)
    \xi_2 \sqrt{\beta_2} \overleftarrow{\vepsilon_2}.
\end{eqnarray*}
Now, let us simplify the first term involving $\overleftarrow{\vx_T}$. Substituting the coefficients of drift, i.e.,  $\omega_t = \omega = \frac{1}{\bnu \sqrt{2}^{1/T}}$, it simplifies to $\frac{1}{\sqrt{2}} \left( \prod_{T-1}^{1} \aat \dm\right) \aat\overleftarrow{\vx_T}$. Similarly, by collecting the terms involving $\rxz$, we get
$\left(\sum_{t=1}^{T-1} \left( \frac{1}{\sqrt{2}}\right)^{t/T} \sqrt{\Bar{\alpha}_t}  \left(\mA\mA^T \mD(\vm) \right)^{t-1} \right) \times \mA\mA^T \mD(\1-\vm) \rxz$. Now, we invoke the assumption that $\{\overleftarrow{\vepsilon_t}\}_{t=2}^{T}$ and $\{\overleftarrow{\vepsilon_t}\}_{t=1}^{T-1}$ are IID Gaussian random variables. Let us denote the combined variance of dispersion terms containing $\overrightarrow{\vepsilon_t}$ in the forward process as $\overrightarrow{\Bar{\Sigma}}$ and terms containing $\overleftarrow{\vepsilon_t}$ in the reverse process as $\overleftarrow{\Bar{\Sigma}}$. Therefore, the final sample generated at the end of the reverse process with inpainting at each intermediate state becomes:
\begin{eqnarray*}
    \overleftarrow{\vx_0} 
    =
    \frac{1}{\sqrt{2}} \left( \prod_{T-1}^{1} \aat \dm\right) \aat\overleftarrow{\vx_T}
    +
    \left(\sum_{t=1}^{T-1} \left( \frac{1}{\sqrt{2}}\right)^{t/T} \sqrt{\Bar{\alpha}_t}  \left(\mA\mA^T \mD(\vm) \right)^{t-1} \right) 
    \\
    \times \mA\mA^T \mD(\1-\vm) \rxz
    +
    \overrightarrow{\Bar{\Sigma}} \overrightarrow{\vepsilon}
    + \overleftarrow{\Bar{\Sigma}} \overleftarrow{\vepsilon},
\end{eqnarray*}
where $\leps$ and $\reps$ are IID Gaussians with $\gauss$. One may wish to combine $\leps$ and $\reps$ to represent as a single Gaussian random variable. We keep them separate to explicitly highlight their contributions. This finishes the proof. \hfill $\square$

\section{Experiments}
\label{sec:exps}


In this section, we provide empirical evidence to support our theoretical results. Since this paper aims to provide a theoretical justification, we only conduct toy experiments just to verify the main theoretical claims. We consider a setting where perfect recovery is possible. By perfect recovery, we mean a Root Mean Squared Error (RMSE) as negligible as below $10^{-5}$. For large-scale experiments using diffusion-based image inpainting, we refer to a recent work by \citet{lugmayr2022repaint}.
\subsection{Implementation Details}
\label{subsec:impl-det}
As shown in Figure~\ref{fig:ddpm-data}, the data generating distribution $q\left(\rxz\right)$ is supported on a linear manifold in $\R^2$, i.e., $k=1$ and $d=2$. We choose $\mA = \begin{bmatrix}
    2/\sqrt{13}\\3/\sqrt{13}
\end{bmatrix}$ and $\vm = \begin{bmatrix}
    0\\1
\end{bmatrix}$, satisfying \textbf{Assumption~\ref{assm:ortho}}~and~\textbf{\ref{assm:inpainting}}. The samples $\rxz \sim q(\rxz)$ follow the structure: $ \rxz = \mA z_0$, where $z_0 \sim \gN\left(0,1\right)$. In Figure~\ref{fig:ddpm-marg}, the marginals over the first and second coordinates of $\rxz$ are given by blue and orange bars, respectively. We consider $n=1000$ samples, variance $\beta = 0.9$, and $R=100$ resampling steps.

\begin{figure}
     \centering
     \begin{subfigure}[b]{0.38\columnwidth}
         \centering
         \includegraphics[width=\linewidth]{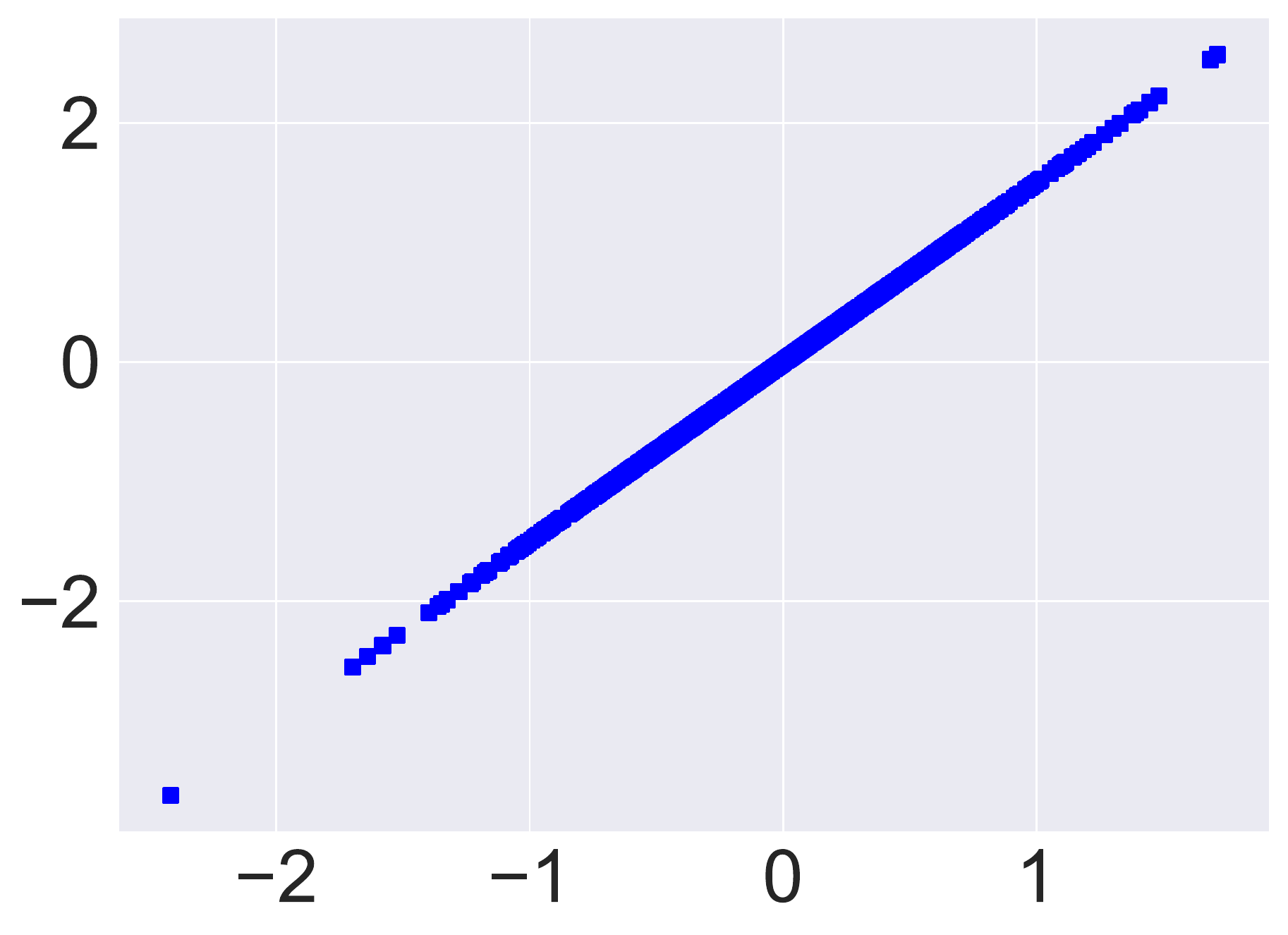}
         \caption{Data manifold}
         \label{fig:ddpm-data}
     \end{subfigure}
     \begin{subfigure}[b]{0.37\columnwidth}
         \centering
         \includegraphics[width=\linewidth]{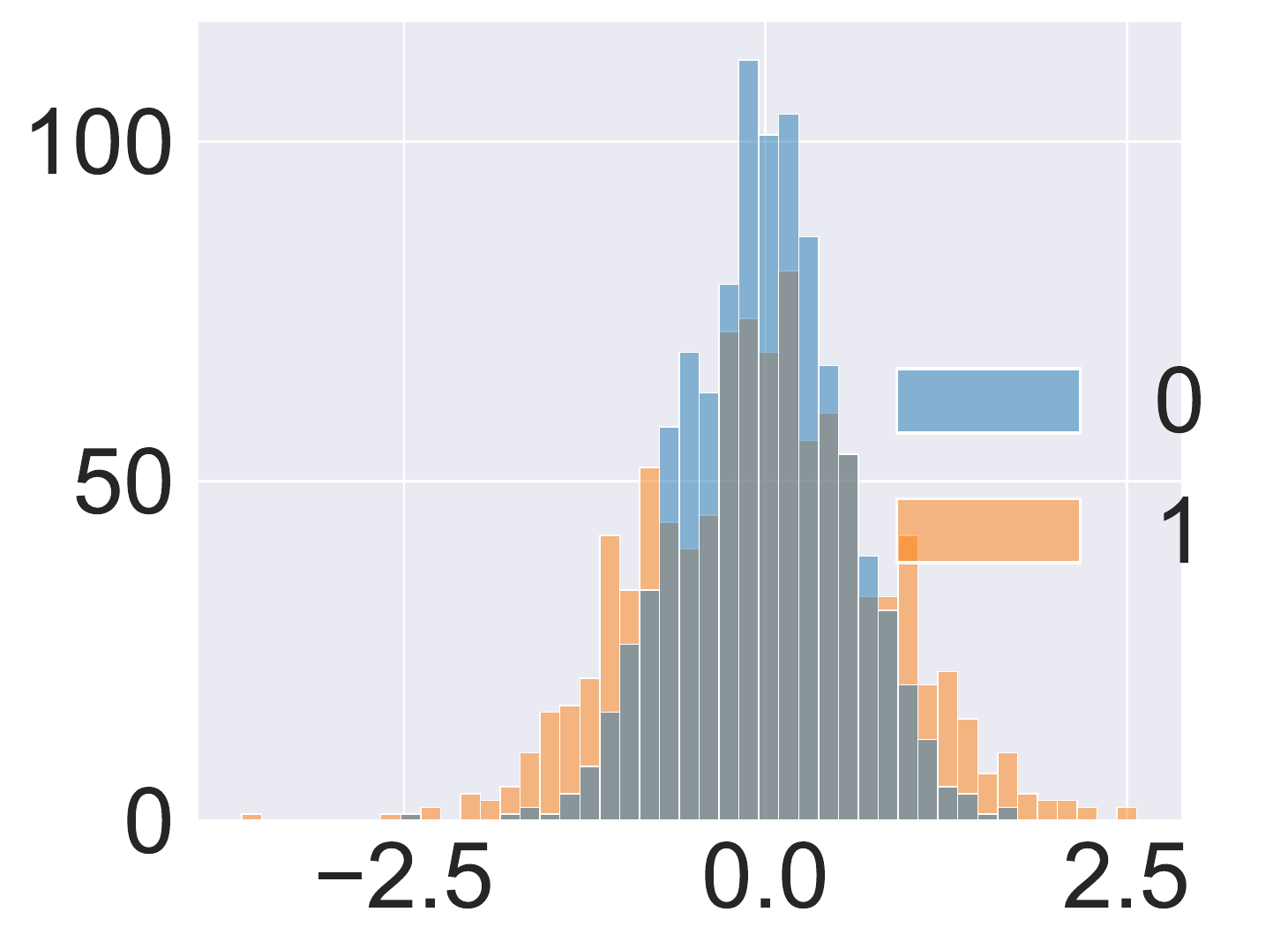}
         \caption{Marginals}
         \label{fig:ddpm-marg}
     \end{subfigure}
     \caption{The data generating distribution is supported on a linear manifold with Gaussian marginals.}
     \label{fig:ddpm-data-marg}
\end{figure}

\subsection{Experimental Results}
\label{subsec:emp-res}

\begin{figure}
     \centering
     \begin{subfigure}[b]{0.45\columnwidth}
         \centering
         \includegraphics[width=\linewidth]{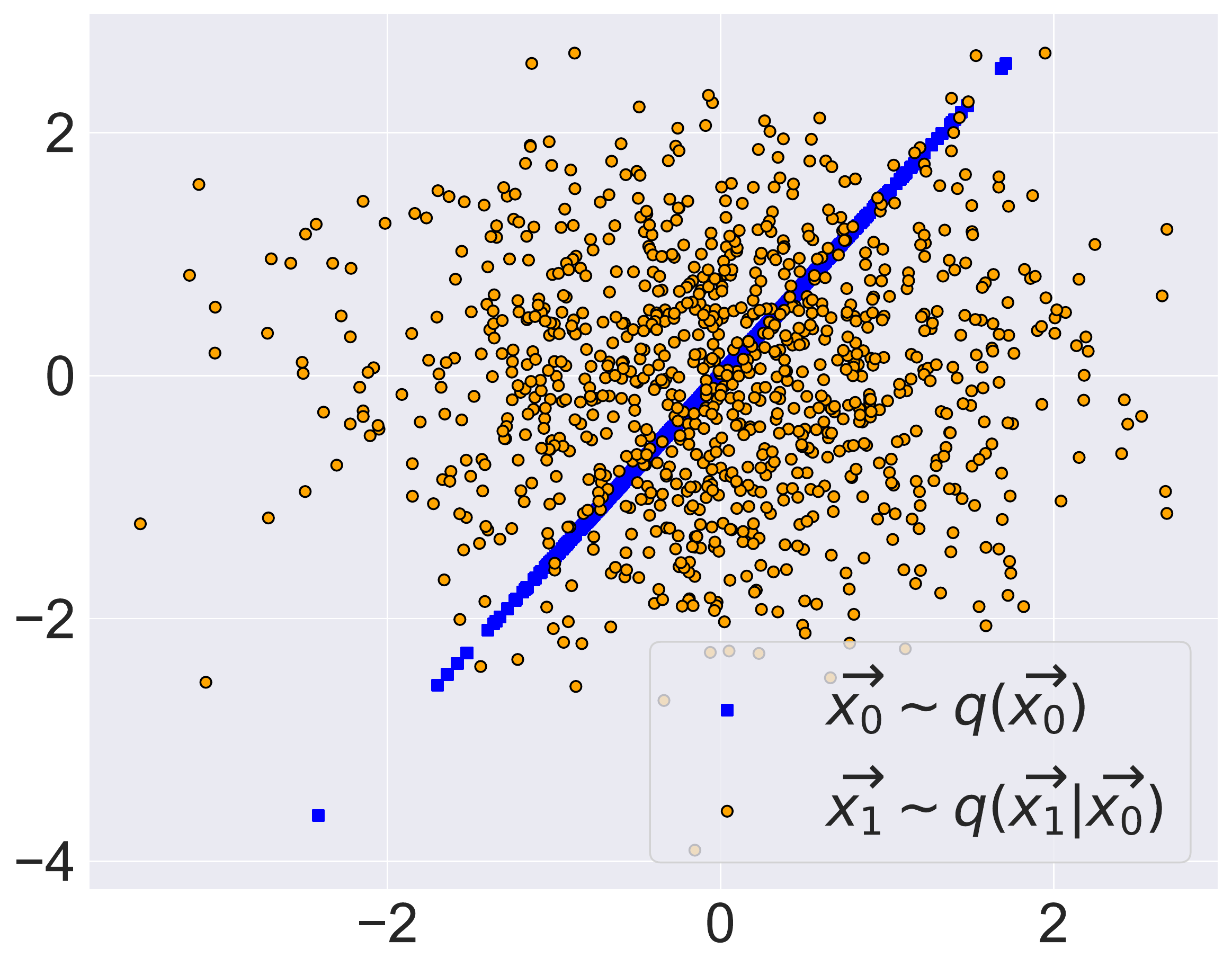}
         \caption{Forward SDE}
         \label{fig:fwd-sde}
     \end{subfigure}
     \begin{subfigure}[b]{0.44\columnwidth}
         \centering
         \includegraphics[width=\linewidth]{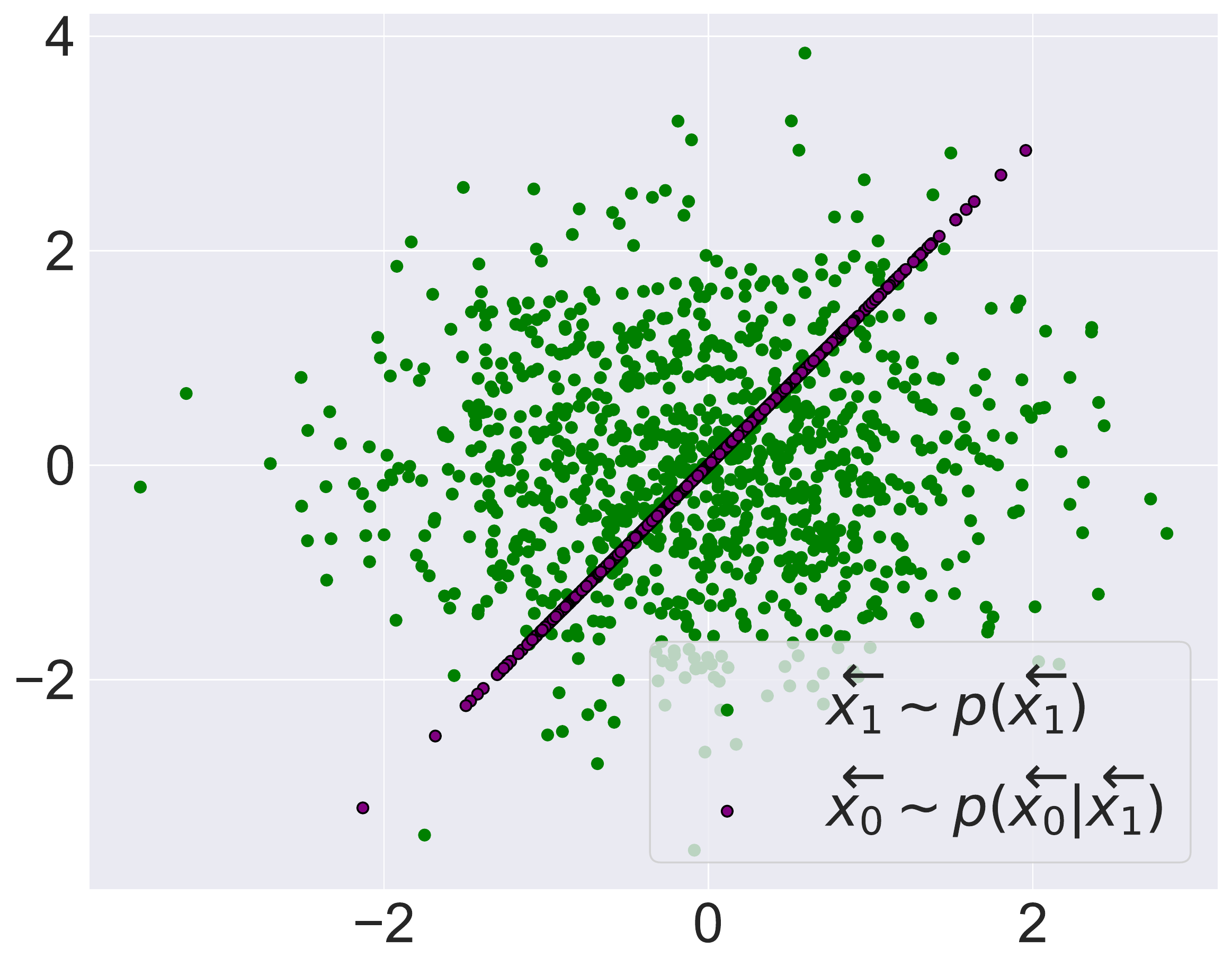}
         \caption{Reverse SDE}
         \label{fig:rev-sde}
     \end{subfigure}
     \caption{Visualization of the forward and reverse processes in DDPM. The forward SDE pushes samples drawn from the data manifold towards the support of Gaussian prior. The reverse SDE discovers the subspace underneath data distribution.}
     \label{fig:ddpm-fwd-rev}
\end{figure}

\begin{wraptable}{r}{5.5cm}
\caption{RMSE of inpainting.}
\label{sample-table}
\vskip 0.15in
\begin{center}
\begin{small}
\begin{sc}
\begin{tabular}{lcr}
\toprule
Method & $d=2,k=1$   \\
\midrule
RePaint    & 22.97 \\
RePaint+RevSDE & 27.61\\
RePaint$^+$ & $3.28 \times 10^{-8}$\\
\bottomrule
\end{tabular}
\end{sc}
\end{small}
\end{center}
\vskip -0.2in
\end{wraptable}
Figure~\ref{fig:fwd-sde} shows $n=1000$ samples drawn from the data generating distribution $q\left( \rxz\right)$ and the output of the forward SDE $\rxo$. This verifies that the forward SDE converges to the reference prior. Figure~\ref{fig:rev-sde} illustrates $n=1000$ samples drawn from the Gaussian prior $p\left( \lxo\right)$ and succesfully projected onto the data manifold by the reverse SDE. In Figure~\ref{fig:repaintplus}, we compare the signals recovered by the baseline algorithm RePaint (\textcolor{brown}{$\bullet $}) and the proposed algorithm RePaint$^+$ (\textcolor{orange}{$*$}). The orignal samples (\textcolor{blue}{$\blacksquare$}) and the inpainted samples by RePaint$^+$ are $\epsilon$-close with an RMSE of $3.28\times 10^{-8}$, as given in Table~\ref{sample-table}. On the other hand, the inpainted samples by RePaint have an RMSE of $22.97$. As discussed in our theoretical analysis (\wasyparagraph{\ref{subsec:main-d-dim}), the bias due to misalignment
forces RePaint to drift away from the data manifold.

One might ask whether passing these inpainted samples through another reverse SDE step would lead to perfect recovery. This modified setup of RePaint projects the generated samples $\lxz$ onto the data manifold, as shown by red arrows in Figure~\ref{fig:repaint-sde}. However, this does not push $\lxz$ into the tiny subspace around $\rxz$ because they still incur an RMSE of $27.61$, which is worse than the original setup. Also, it is clear from Figure~\ref{fig:repaint-sde} that even the least square projections onto the data manifold does not recover the original samples. Therefore, there exists a non-trivial bias in the original RePaint algorithm that adversely affects perfect recovery. Contrary to that, RePaint$^+$ eliminates this bias and enjoys perfect recovery with infinitely many resampling steps. For a finite number of steps, it converges to a tiny $\epsilon$-neighborhood of the original samples. This strengthens our algorithmic insights drawn from the analysis in the main paper. 
\vspace{-0.2in}
\begin{figure}[!b]
\begin{center}
\centerline{\includegraphics[width=0.45\columnwidth]{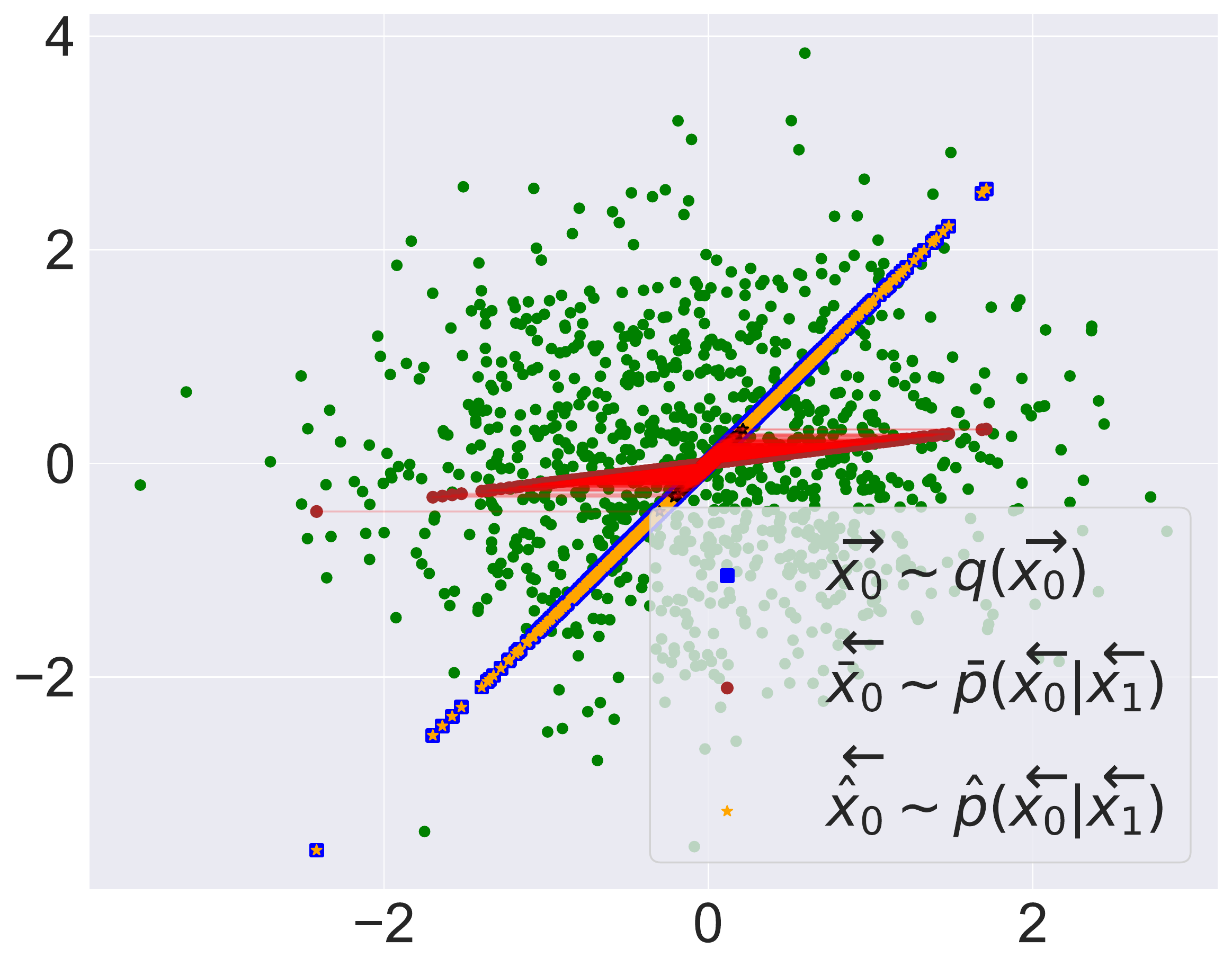}}
\caption{Running reverse SDE on top of the inpainted samples generated by RePaint. The final reverse step is highlighted by red arrows (\textcolor{red}{$\rightarrow$}). It fails to recover the true underlying sample.} 
\label{fig:repaint-sde}
\end{center}
\vskip -0.2in
\end{figure}



\end{document}